\documentclass{article} 
\usepackage{iclr2026_conference,times}

\usepackage{amsmath,amsfonts,bm}

\def\eqref#1{equation~\ref{#1}}

\def\1{\bm{1}}

\def\vv{{\bm{v}}}

\DeclareMathAlphabet{\mathsfit}{\encodingdefault}{\sfdefault}{m}{sl}
\SetMathAlphabet{\mathsfit}{bold}{\encodingdefault}{\sfdefault}{bx}{n}

\newcommand{\E}{\mathbb{E}}

\newcommand{\R}{\mathbb{R}}

\DeclareMathOperator*{\argmin}{arg\,min}

\usepackage{hyperref}
\usepackage{url}
\usepackage{amsmath,amssymb,amsthm}
\usepackage{booktabs}
\usepackage{graphicx}
\usepackage{xcolor}
\usepackage{enumitem}
\usepackage{pifont}
\usepackage{algorithm}
\usepackage{algorithmic}

\newtheorem{theorem}{Theorem}
\newtheorem{proposition}{Proposition}
\newtheorem{lemma}{Lemma}
\newtheorem{assumption}{Assumption}
\providecommand{\x}{\mathbf{x}}
\providecommand{\y}{\mathbf{y}}
\providecommand{\z}{\mathbf{z}}
\providecommand{\uu}{\mathbf{u}}
\providecommand{\s}{\mathbf{s}}
\providecommand{\g}{\mathbf{g}}

\title{What Do Flow-Based Inverse Solvers Approximate?\\ A Posterior-Transport View}

\iclrfinalcopy

\author{Jian Xu$^{1,2}$, Delu Zeng$^{3}$, John Paisley$^{4}$, Qibin Zhao$^{2}$ \\
$^{1}$RIKEN iTHEMS \quad $^{2}$RIKEN AIP \quad $^{3}$South China University of Technology \quad $^{4}$Columbia University \\
\texttt{jian.xu@riken.jp}}

\begin{document}

\maketitle
\thispagestyle{fancy}\lhead{Preprint}\fancyhead[L]{Preprint}

\begin{abstract}
A growing family of training-free solvers---FlowDPS, FLOWER, PnP-Flow and their
diffusion ancestors (DPS, DAPS)---repurpose a pretrained flow-matching prior to
solve imaging inverse problems by adding a measurement-guidance term to the
deterministic probability-flow ODE. Despite strong empirical results, what these
per-step corrections actually approximate---and how far the resulting samples are
from the true posterior $p(\x\mid\y)$---has not been characterized. We give a
posterior-transport account of flow-based inverse problem solving. Our starting
point is a simple but consequential fact: for a \emph{deterministic} flow prior,
Bayesian conditioning is realized entirely by a \emph{reweighting of the source
distribution}, not by a drift correction; pushing the reweighted source through
the \emph{unmodified} velocity field yields exact posterior samples. From this we
show that trajectory-guidance solvers can be read as the minimum-kinetic-energy
\emph{correction} field needed to morph the unconditional source into the
posterior, and that FlowDPS / FLOWER / PnP-Flow correspond to distinct
zeroth-order / Gaussian / proximal approximations of this single object; we bound
the resulting posterior bias in Wasserstein distance. A controlled $2$D study with
a closed-form posterior confirms the theory decisively: source reweighting matches
the true posterior to the Monte-Carlo floor on every metric, whereas trajectory
guidance incurs $200$--$800\times$ larger error and collapses posterior modes,
\emph{regardless of guidance strength}. Guided by the analysis we propose a cheap,
principled velocity-correction solver that is competitive across two in-domain priors (AFHQ, CelebA) and two out-of-distribution
settings while, unlike point-estimate source-space optimizers,
producing diverse posterior samples with uncertainty that correlates with reconstruction error.
\end{abstract}

\section{Introduction}

Flow matching \citep{lipmanflow,liuflow,albergo2022building} has
become a state-of-the-art framework for generative modeling, learning a velocity
field $\vv_t$ whose probability-flow ODE transports a simple source $p_0$ to a
data distribution $p_1$. A pretrained flow model is an implicit prior over images,
and a large recent literature reuses it---without retraining---to solve linear and
nonlinear inverse problems such as super-resolution, deblurring and inpainting
\citep{kim2025flowdps,pourya2026flower,martin2025pnp,ben2024d}. These
solvers all share a template inherited from diffusion posterior sampling
\citep{chung2022diffusion,zhang2025improving}: integrate the prior flow ODE and, at each
step, nudge the trajectory toward the measurement $\y$ with a likelihood-based
correction.

This template works well empirically, but it is built bottom-up by analogy to
diffusion guidance, and a basic question has gone unanswered: \emph{what does the
per-step correction approximate, and how biased is the resulting sampler relative
to the true posterior $p(\x\mid\y)$?} Perceptual metrics (LPIPS, FID) reward
producing \emph{a} plausible reconstruction, and so they mask the distinction
between sampling the posterior and merely landing on its typical set. For
reliability-critical restoration---and for any use of these solvers to quantify
uncertainty---the distinction matters.

We take a top-down, posterior-transport view. The key observation is that the
\emph{deterministic} nature of the probability-flow ODE changes the structure of
Bayesian conditioning relative to the stochastic (diffusion) case. For an SDE
prior, conditioning on $\y$ adds a Doob $h$-transform drift to the dynamics
\citep{denker2024deft}. For a deterministic flow, we show there is \emph{no
canonical added drift}: the exact posterior is obtained by leaving the velocity
field untouched and \emph{reweighting the source distribution} by the endpoint
likelihood (Section~\ref{sec:exact}, Proposition~\ref{prop:dichotomy}). Conditioning lives
in the source, not in the trajectory.

This reframing has three consequences. (i) It explains what trajectory-guidance
solvers are doing: because sampling the reweighted source is intractable, they
start from the \emph{unconditional} source and inject a correction field to morph
it into the posterior; we identify the canonical (minimum-kinetic-energy)
correction and show FlowDPS, FLOWER and PnP-Flow are distinct approximations of it
(Proposition~\ref{thm:approx}). (ii) It yields a posterior-bias bound in Wasserstein
distance controlled by the correction-estimation error (Theorem~\ref{thm:bound}).
(iii) It tells us where existing solvers must fail---multimodal posteriors, where
a local correction collapses to a single mode---and what a faithful solver should
instead do.

We validate the theory on a $2$D problem with a closed-form Gaussian-mixture
posterior, removing all model error: source reweighting reaches the Monte-Carlo
floor on sliced-Wasserstein, energy distance, MMD and mode-weight recovery, while
trajectory guidance is $200$--$800\times$ worse and collapses modes for
\emph{every} guidance strength (Section~\ref{sec:toy}). On AFHQ, CelebA, and
out-of-distribution image restoration we reproduce the standard benchmark and, guided by the analysis, give a cheap
velocity-correction solver (no ODE backpropagation) that is competitive on quality
while producing diverse posterior samples with error-correlated uncertainty that point-estimate
source-space optimizers (e.g.\ D-Flow) cannot (Section~\ref{sec:imaging}).

\paragraph{Contributions.}
\begin{itemize}[leftmargin=1.5em,itemsep=0.1em]
\item \textbf{Exact posterior transport for flow priors.} We prove that, for a
deterministic flow prior, the posterior is the pushforward of a reweighted source
under the \emph{unmodified} velocity field---conditioning is source reweighting,
not drift (Proposition~\ref{prop:invariant}).
\item \textbf{A unifying view of flow inverse solvers.} We characterize the
canonical correction field that trajectory-guidance solvers approximate and place
FlowDPS, FLOWER, PnP-Flow and DPS within it as zeroth-order / Gaussian / proximal
approximations (Proposition~\ref{thm:approx}), with a Wasserstein posterior-bias bound
(Theorem~\ref{thm:bound}).
\item \textbf{Decisive empirical confirmation and a practical solver.} A
closed-form $2$D study confirms the predictions exactly; on image restoration our
analysis-guided solver matches strong baselines in quality while delivering
error-correlated posterior uncertainty.
\end{itemize}

\section{Background and problem setup}
\label{sec:bg}

\paragraph{Flow-matching priors.} A flow-matching model learns a time-dependent
velocity field $\vv_t:\R^d\to\R^d$, $t\in[0,1]$, such that the probability-flow ODE
\begin{equation}
\frac{d\x_t}{dt} = \vv_t(\x_t), \qquad \x_0\sim p_0=\mathcal{N}(0,I),
\label{eq:pfode}
\end{equation}
transports the source $p_0$ to the data distribution $p_1\!\approx\! p_{\mathrm{data}}$.
Let $\Phi_{s\to t}$ denote the flow map of \eqref{eq:pfode} (the solution operator
from time $s$ to $t$); it is a deterministic diffeomorphism, and
$p_t=(\Phi_{0\to t})_\# p_0$. We write $\Phi:=\Phi_{0\to1}$, so $p_1=\Phi_\# p_0$.

\paragraph{Inverse problems.} We observe
$\y = \mathcal{A}(\x)+\boldsymbol{\eta}$, $\boldsymbol{\eta}\sim\mathcal{N}(0,\sigma^2 I)$,
with a known (possibly non-invertible) forward operator $\mathcal A$ and likelihood
$p(\y\mid\x)$. Using the flow model as prior, the Bayesian posterior is
$p(\x\mid\y)\propto p(\y\mid\x)\,p_1(\x)$. We want samples from $p(\x\mid\y)$, not
just a single reconstruction.

\paragraph{Trajectory-guidance solvers.} The dominant template integrates
\eqref{eq:pfode} from $\x_0\sim p_0$ while adding a measurement correction:
$\tfrac{d\x_t}{dt}=\vv_t(\x_t)+\uu_t(\x_t)$, where $\uu_t$ is built from a one-step
endpoint estimate $\hat\x_1(\x_t)=\x_t+(1-t)\vv_t(\x_t)$ (a flow Tweedie formula)
and the likelihood gradient $\nabla_{\x_t}\log p(\y\mid\hat\x_1)$. FlowDPS
\citep{kim2025flowdps} uses this gradient directly; FLOWER \citep{pourya2026flower}
replaces it with an isotropic-Gaussian endpoint-posterior projection solved by
conjugate gradients; PnP-Flow \citep{martin2025pnp} alternates a data-fidelity
step with flow-path reprojection and denoising. D-Flow \citep{ben2024d} is the
exception: it optimizes the \emph{source} latent $\z$ to fit $\y$, returning a
single MAP-like point.

\section{Exact posterior transport under a flow prior}
\label{sec:exact}

We first establish that, for a deterministic flow prior, conditioning is a source
operation. Throughout, $\Phi$ is the prior flow map and $p_0=\mathcal N(0,I)$.

\begin{lemma}[Posterior = reweighted source; classical]
\label{thm:exact}
Define the \emph{source posterior}
\begin{equation}
p_0^\y(\z) \;\propto\; p\big(\y\mid \Phi(\z)\big)\,p_0(\z).
\label{eq:source-posterior}
\end{equation}
Then $\x=\Phi(\z)$ with $\z\sim p_0^\y$ is distributed exactly as $p(\x\mid\y)$,
i.e.\ $p(\cdot\mid\y)=\Phi_\# p_0^\y$.
\end{lemma}

Lemma~\ref{thm:exact} is the elementary change-of-variables identity underlying
latent-space Bayesian inference and normalizing-flow posterior inference; we claim
no novelty for it and state it only to fix notation. Our contribution is what it
implies for the \emph{dynamics}, which is specific to flow matching and, to our
knowledge, has not been made explicit. Define the time-$t$ reweighting
$r_t(\x):=p(\y\mid\Phi_{t\to1}(\x))/p(\y)$.

\begin{proposition}[Conditioning is a prior-flow invariant]
\label{prop:invariant}
The posterior marginal is the reweighted prior marginal, $p_t^\y=r_t\,p_t$, and
$r_t$ is conserved along prior-flow characteristics, $\partial_t r_t+\vv_t\!\cdot\!\nabla r_t=0$.
Consequently the posterior marginal path is transported by the \emph{unmodified}
prior velocity,
\begin{equation}
\partial_t p_t^\y + \nabla\!\cdot\!\big(p_t^\y\,\vv_t\big)=0,\qquad p_1^\y=p(\cdot\mid\y),
\label{eq:posterior-transport}
\end{equation}
i.e.\ no drift correction is required (proof in Appendix~\ref{app:proofs}).
\end{proposition}

The conservation $\tfrac{D}{Dt}r_t=0$ holds because the endpoint
$\Phi_{t\to1}(\x_t)=\x_1$ is constant along a deterministic characteristic: the
``filtering'' expectation that defines the diffusion $h$-function collapses to a
point evaluation. This yields our central conceptual point, which sharply separates
the flow case from the diffusion case it is usually derived by analogy to.

\begin{proposition}[Zero-drift dichotomy / small-noise $h$-transform limit]
\label{prop:dichotomy}
For the family of noisy bridges
$d\x_t=[\vv_t(\x_t)+\epsilon^2\nabla\log h_t^\epsilon(\x_t)]\,dt+\epsilon\,dW_t$
conditioned on $\y$ at $t=1$, the Doob drift correction
$\epsilon^2\nabla\log h_t^\epsilon$ vanishes pointwise as $\epsilon\to0$, while the
endpoint constraint is enforced in the limit through the reweighted initial law
$p_0^\y$. Hence in the deterministic ($\epsilon=0$) flow limit, conditioning carries
\emph{no} drift and is realized entirely by source reweighting (proof in Appendix~\ref{app:proofs}).
\end{proposition}

Proposition~\ref{prop:dichotomy} is the conceptual core: porting a diffusion
guidance \emph{drift} (a Doob $h$-transform, as in classifier/measurement guidance
and \citet{denker2024deft}) onto a deterministic flow is structurally mismatched,
because in the flow limit there is no drift to port---the conditioning information
has migrated into the source. This explains, rather than merely observes, why
trajectory-guidance solvers cannot be exact in general (Section~\ref{sec:unify}).

\paragraph{Why solvers add a drift anyway.} Proposition~\ref{prop:invariant} also exposes
the catch: it requires \emph{initializing at the reweighted source} $p_0^\y$. One
can sample $p_0=\mathcal N(0,I)$, but $p_0^\y$ is intractable (it requires the
full flow map inside the likelihood), and naive importance reweighting of $p_0$
suffers from vanishing effective sample size in high dimension. Trajectory-guidance
solvers sidestep this by starting from the \emph{unconditional} $p_0$ and adding a
correction field $\uu_t$ to bend the marginals from $p_t$ to $p_t^\y$. The next
section makes this precise.

\section{Flow inverse solvers as approximate posterior corrections}
\label{sec:unify}

Suppose we integrate the corrected ODE
$\tfrac{d\x_t}{dt}=\vv_t(\x_t)+\uu_t(\x_t)$ from $\x_0\sim p_0$ and let $q_t^{\uu}$
be the resulting marginals. We seek $\uu_t$ that lands on the posterior,
$q_1^{\uu}=p(\cdot\mid\y)$. Among all such fields, the canonical choice is the one
of minimum kinetic energy:
\begin{equation}
\uu_t^\star=\argmin_{\uu}\;\int_0^1\!\E_{q_t^{\uu}}\big\|\uu_t(\x_t)\big\|^2\,dt
\quad\text{s.t.}\quad q_1^{\uu}=p(\cdot\mid\y).
\label{eq:ustar}
\end{equation}
This is a Benamou--Brenier / Schr\"odinger-type control problem: $\uu^\star$ is the
optimal-transport correction taking the prior path to the posterior path. It is the
deterministic analogue of the Doob/F\"ollmer drift, but defined through a chosen
mobility (here the Euclidean metric) rather than a diffusion coefficient.

We stress what is and is not being claimed. \emph{Any} field with
$q_1^{\uu}=p(\cdot\mid\y)$ is a valid corrector; $\uu^\star$ is merely the
\emph{canonical} (minimum-energy) representative, and Theorem~\ref{thm:bound} bounds
the bias of \emph{any} corrector by its discrepancy from $\uu^\star$. Existing
solvers do not attempt to compute $\uu^\star$; they follow a likelihood gradient and
are \emph{greedy, local} correctors that in general do \emph{not} reach the posterior
at all. So Proposition~\ref{thm:approx} should be read not as ``they approximate the
single object $\uu^\star$'' but as ``they are local likelihood-gradient correctors
whose gap to any exact corrector is what Theorem~\ref{thm:bound} controls and what
Section~\ref{sec:toy} measures.''

\begin{proposition}[Existing solvers as approximations of $\uu^\star$]
\label{thm:approx}
Write the generic correction as
$\uu_t(\x)=M_t(\x)\,\nabla_{\x}\log p\big(\y\mid \hat\x_1(\x,t)\big)$ for a mobility
operator $M_t$ and endpoint estimate $\hat\x_1$. Then:
\begin{enumerate}[leftmargin=1.5em,itemsep=0.1em]
\item \textbf{DPS / FlowDPS} use $\hat\x_1=\x+(1-t)\vv_t(\x)$ (flow Tweedie) and a
scalar mobility $M_t=\zeta_t I$, i.e.\ a zeroth-order (point-mass endpoint)
approximation of $\uu^\star$;
\item \textbf{FLOWER} replaces the point mass by an isotropic-Gaussian endpoint
posterior $\mathcal N(\hat\x_1,\sigma_{r,t}^2 I)$ and sets $M_t$ to the
corresponding Gaussian posterior-covariance (a $\Pi$GDM/CG projection);
\item \textbf{PnP-Flow} replaces the gradient step by a data-fidelity proximal map
followed by flow-path reprojection, a splitting approximation of the same
correction.
\end{enumerate}
All three neglect higher-order structure of the endpoint posterior
$p(\x_1\mid\x_t)$ and therefore incur a nonzero gap $\uu_t-\uu_t^\star$ even with
exact velocity fields.
\end{proposition}

\begin{theorem}[Posterior-bias bound]
\label{thm:bound}
Let $p_1^{\uu}$ be the terminal law of the corrected ODE with an approximate field
$\uu$, and let $p(\cdot\mid\y)$ be the true posterior. Under Lipschitz regularity
of $\vv_t+\uu_t$ (Assumption~\ref{ass:lip}, Appendix~\ref{app:proofs}),
\begin{equation}
W_2\big(p_1^{\uu},\,p(\cdot\mid\y)\big)\;\le\;
L\!\int_0^1\!\big(\E_{q_t^{\uu}}\|\uu_t-\uu_t^\star\|^2\big)^{1/2}dt
\;+\;\epsilon_{\mathrm{disc}}+\epsilon_{\mathrm{flow}},
\label{eq:bound}
\end{equation}
where $\epsilon_{\mathrm{disc}}$ is ODE discretization error and
$\epsilon_{\mathrm{flow}}$ is the flow-model approximation error.
\end{theorem}

Equation~\eqref{eq:bound} is a stability guarantee: the posterior bias of any
trajectory-guidance solver is controlled by how well its correction matches
$\uu^\star$. The optimal corrector $\uu^\star$ is generally intractable, so we do
not claim \eqref{eq:bound} gives a computable error estimate; rather it identifies
the right object (the correction--$\uu^\star$ gap), whose magnitude we then measure
directly against a closed-form posterior in Section~\ref{sec:toy}. That measurement
shows the gap is substantial for likelihood-gradient guidance and, importantly, does
\emph{not} vanish under any scalar guidance strength---consistent with a
mis-specified correction direction rather than a mis-tuned magnitude.

\subsection{A large-deviation view of source reweighting and guidance collapse}
\label{sec:ldp}

The source representation \eqref{eq:source-posterior} has a clean small-noise
asymptotic geometry that sharpens \emph{why} source reweighting is exact yet hard,
and why guidance collapses. We take the \emph{small-observation-noise} regime, in
which only the measurement noise shrinks, $\y=A(\x)+\sqrt{\epsilon}\,\boldsymbol\eta$,
$\boldsymbol\eta\sim\mathcal N(0,I)$, and the \emph{source prior is kept fixed}
($p_0=\mathcal N(0,I)$, matching the actual model---we do \emph{not} temper the
prior). Writing $\ell_\y(\x)=\tfrac12\|A(\x)-\y\|^2$, the source posterior is the
tilted measure
\begin{equation}
p_0^{\y,\epsilon}(\z)\;\propto\;\exp\!\big(-\ell_\y(\Phi(\z))/\epsilon\big)\,p_0(\z).
\label{eq:rate}
\end{equation}

\begin{proposition}[Small-noise concentration and Laplace weights]
\label{prop:ldp}
Let $\ell_\y\!\circ\Phi$ be $C^2$ with isolated nondegenerate minimizers
$\{\z_k^\star\}$ on $\mathrm{supp}\,p_0$ (the measurement-consistent set). Then as
$\epsilon\to0$, $p_0^{\y,\epsilon}$ concentrates on $\{\z_k^\star\}$ and, by Laplace's
method, converges weakly to the categorical mixture $\sum_k w_k\,\delta_{\z_k^\star}$
with basin weights
\begin{equation}
w_k\;\propto\; p_0(\z_k^\star)\,\big|\nabla^2(\ell_\y\!\circ\Phi)(\z_k^\star)\big|^{-1/2},
\end{equation}
the prior entering through the prefactor $p_0(\z_k^\star)$. Since $\Phi$ is a
diffeomorphism, the data posterior $p^{\y,\epsilon}_1=\Phi_\#p_0^{\y,\epsilon}$
concentrates on $\{\Phi(\z_k^\star)\}$ with the same weights.
\end{proposition}

\emph{Regime and caveat.} Proposition~\ref{prop:ldp} is an \emph{asymptotic,
explanatory} statement about the small-observation-noise limit with a fixed prior; it
is not a claim about the fixed-finite-$\sigma$ posterior, and the basin picture is
exact only as $\epsilon\to0$. We use it to explain---not to certify---the finite-noise
phenomena observed empirically (ESS collapse, mode mis-weighting). The
path-space/Freidlin--Wentzell statements below additionally use a small \emph{dynamics}
noise (a reference SDE); these are likewise asymptotic.

Proposition~\ref{prop:ldp} (proof in Appendix~\ref{app:proofs}) makes three points
precise. \textbf{(i) Why source reweighting is exact but hard.} The exact estimator
weights $p_0$ by $e^{-\ell_\y(\Phi(\z))/\epsilon}$, but the rate function localizes
mass in $O(\sqrt{\epsilon})$ neighborhoods of $\{\z_k^\star\}$; unconditional draws
$\z\sim p_0$ rarely hit them, so the effective sample size collapses---a
large-deviation account of the high-dimensional ESS failure noted above.
\textbf{(ii) Why guidance mis-weights modes.} The true posterior is a Laplace
\emph{mixture} over basins with weights $w_k$. Source reweighting preserves these
inter-basin weights, whereas a likelihood-gradient correction reweights them
incorrectly because it transports mass by local ascent rather than by the global
optimal coupling. We make this precise only for an \emph{idealized caricature}---a
pure single-trajectory ascent of $\nabla\log r_t$, which collapses to one basin and
gives $\mathrm{TV}\ge 1-\max_k w_k$ independent of guidance strength
(Proposition~\ref{prop:collapse}, Appendix~\ref{app:ldp-ext}). Real solvers integrate
$\vv_t+\uu_t$, and the prior velocity $\vv_t$ itself spreads mass across modes, so
their failure is \emph{softer}: mode mis-weighting rather than exact single-basin
collapse. The idealized bound is therefore an upper caricature of the effect; what we
verify empirically is the milder but real mis-weighting (Table~\ref{tab:toy}: large
mode-weight error and a non-vanishing distributional gap that no guidance strength
removes) and, complementarily, the exponential ESS decay of exact source IS
(Proposition~\ref{prop:ess}; Figure/Table in Appendix~\ref{app:ldp-ext}). \textbf{(iii) What $\uu^\star$ is.}
The minimum-kinetic-energy correction \eqref{eq:ustar} is the deterministic-limit
minimum-action (Freidlin--Wentzell) control \citep{kifer1988random,dembo2009large}
steering the \emph{unconditional} source to the posterior; from the reweighted source
$p_0^\y$ no action is needed (Proposition~\ref{prop:invariant}). This situates our
transport picture within the Schr\"odinger-bridge / F\"ollmer-drift continuum
(Appendix~\ref{app:relation}).

\section{A controlled study with a closed-form posterior}
\label{sec:toy}

\begin{figure}[t]
\centering
\includegraphics[width=\linewidth]{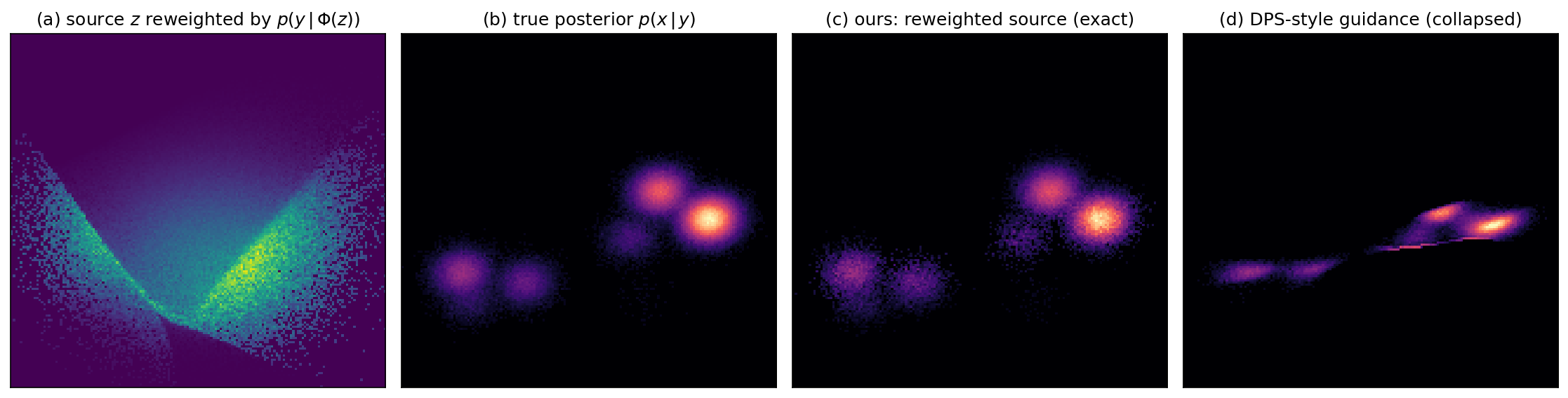}
\caption{\textbf{Conditioning lives in the source} (densities, $8$-mode $2$D GMM
prior, linear measurement). \textbf{(a)} The source $\mathcal N(0,I)$ reweighted by
the endpoint likelihood $p(\y\mid\Phi(\z))$---conditioning is a reweighting on the
source. \textbf{(b)} Closed-form true posterior. \textbf{(c)} Our source reweighting
pushed through the \emph{unmodified} velocity field---visually indistinguishable from
(b). \textbf{(d)} DPS-style trajectory guidance collapses the multimodal posterior
onto a single measurement-consistent ridge.}
\label{fig:toy}
\end{figure}

To test the theory free of model error we use a $2$D flow prior whose target is a
random $8$-component Gaussian mixture; the conditional Euler velocity is available
in closed form, so the flow map is essentially exact. With a linear-Gaussian
measurement the true posterior is again a Gaussian mixture, available in closed
form. We compare against this ground truth: (i) \textbf{source reweighting}
(ours)---integrate the unmodified ODE from $p_0$ and weight by the endpoint
likelihood (Lemma~\ref{thm:exact}); and (ii) \textbf{FlowDPS-style guidance}---
the Tweedie-endpoint likelihood-gradient correction, swept over guidance strength.
We report sliced-Wasserstein ($W_2$), energy distance, multi-scale MMD$^2$, and
mode-weight $\ell_1$ error, against a Monte-Carlo \emph{floor} (two independent
draws of the true posterior).

\begin{table}[t]
\centering
\small
\caption{\textbf{$2$D study against the closed-form posterior} ($K{=}8$ modes,
$20$k samples). Source reweighting reaches the Monte-Carlo floor on every metric;
trajectory guidance is $200$--$800\times$ worse on distributional metrics and badly
misestimates mode weights, for \emph{all} guidance strengths $g$.}
\label{tab:toy}
\begin{tabular}{lcccc}
\toprule
Method & sliced-$W_2$ & energy dist. & MMD$^2$ & mode-weight $\ell_1$ \\
\midrule
NULL floor (GT vs GT) & 0.060 & 0.0004 & $1.5\times10^{-4}$ & 0.024 \\
\textbf{Source reweighting (ours)} & \textbf{0.032} & \textbf{0.0006} & $\mathbf{2.1\times10^{-4}}$ & \textbf{0.028} \\
\midrule
FlowDPS-style guidance, $g{=}0.5$ & 0.71 & 0.122 & $4.4\times10^{-2}$ & 0.53 \\
FlowDPS-style guidance, $g{=}1.0$ & 0.38 & 0.087 & $3.0\times10^{-2}$ & 0.16 \\
FlowDPS-style guidance, $g{=}2.0$ & 0.53 & 0.243 & $7.8\times10^{-2}$ & 0.52 \\
FlowDPS-style guidance, $g{=}4.0$ & 0.84 & 0.307 & $9.5\times10^{-2}$ & 0.56 \\
\midrule
DAPS \citep{zhang2025improving} (decoupled annealing) & 0.87 & 0.076 & $1.5\times10^{-2}$ & 0.29 \\
MGPS \citep{moufad2025variational} (midpoint guidance) & 0.45 & 0.092 & $3.2\times10^{-2}$ & 0.17 \\
\bottomrule
\end{tabular}
\end{table}

\paragraph{Scope of this comparison.} Table~\ref{tab:toy} compares solver
\emph{mechanisms} in a setting where it is fair to do so: all methods act on the
\emph{same} analytic flow prior with a closed-form posterior, so differences reflect
the solver, not the prior. This is the only place we can directly compare against
diffusion-native solvers, which are normally run with a \emph{diffusion} prior on a
different benchmark---so a shared imaging table would confound prior and solver, and
we restrict the imaging tables (Section~\ref{sec:imaging}) to the flow-prior family.
We reproduce the \emph{mechanisms} of DAPS \citep{zhang2025improving} (decoupled annealing
with an exact Gaussian endpoint-posterior inner step, favourable to DAPS) and MGPS
\citep{moufad2025variational} (midpoint guidance) on the shared analytic prior. Both improve
on pure likelihood-gradient guidance (lower energy/MMD), yet remain
$\sim\!130$--$190\times$ the Monte-Carlo floor and still mis-weight modes
($\ell_1=0.17$--$0.29$): even an annealing/MCMC solver that provably targets the
posterior in the limit is, at practical budgets, far from the exact source
reweighting. (We do \emph{not} report a toy number for the optimization-based SITCOM
\citep{alkhouri2025sitcom}: its triple-consistency objective is defined for the
diffusion setup and our attempts to port it to the $2$D flow were unstable; we did
not want to report a strawman. It is discussed in Section~\ref{sec:related}.) These
results are the empirical signature of our thesis across the solver family.

Table~\ref{tab:toy} and Figure~\ref{fig:toy} confirm the theory decisively. Source
reweighting through the unmodified velocity field matches the true posterior to the
sampling floor on all four metrics, validating Proposition~\ref{prop:invariant}: the
conditioning is captured entirely in the source. Trajectory guidance, by contrast,
is two-to-three orders of magnitude worse on every distributional metric and
collapses the multimodal posterior onto a single ridge (Figure~\ref{fig:toy}d);
increasing the guidance strength does not fix this---the error plateaus and
then worsens---exactly the non-vanishing floor predicted by Theorem~\ref{thm:bound}.

\section{Image restoration}
\label{sec:imaging}

\paragraph{Setup.} We use AFHQ-Cat $256{\times}256$ with a pretrained
optimal-transport flow-matching prior, and the standard inverse problems of the
FLOWER benchmark \citep{pourya2026flower}: $\times4$ super-resolution, Gaussian
deblurring, and box inpainting, all at measurement noise $\sigma{=}0.05$. We report
PSNR and LPIPS. All methods share the same pretrained prior.

\paragraph{An analysis-guided solver.} Proposition~\ref{thm:approx} says a faithful
solver should (i) use a well-chosen mobility $M_t$ rather than a fixed scalar
step, and (ii) preserve, rather than collapse, posterior multimodality. We adopt a
cheap velocity-correction scheme that integrates the prior ODE from a \emph{random}
source draw and applies a per-step, mobility-modulated, \emph{normalized}
measurement correction (Appendix~\ref{app:method}); it requires one
forward/backward through the endpoint per step---\emph{no} backpropagation through
the ODE, unlike source-space optimizers such as D-Flow. Different source draws give
different posterior samples, from which we form a posterior mean and a per-pixel
uncertainty map.

\begin{table}[t]
\centering
\small
\caption{\textbf{AFHQ-Cat restoration} (PSNR\,$\uparrow$ / LPIPS\,$\downarrow$).
Our velocity-correction solver attains the best perceptual quality (LPIPS) on box
inpainting and is competitive throughout, while uniquely providing
posterior samples (Table~\ref{tab:uncert}) at low cost (Table~\ref{tab:capab}).
Best / second-best LPIPS per column in \textbf{bold} / \underline{underline}.}
\label{tab:imaging}
\begin{tabular}{lccc}
\toprule
Method & Super-resolution & Gaussian deblur & Box inpainting \\
\midrule
FLOWER \citep{pourya2026flower} & 25.95 / 0.276 & 27.25 / 0.301 & 25.65 / 0.077 \\
OT-ODE & 25.18 / \textbf{0.107} & 26.40 / \textbf{0.117} & 25.23 / 0.083 \\
Flow-Priors & 23.36 / 0.276 & 25.85 / 0.183 & 26.15 / 0.125 \\
D-Flow \citep{ben2024d} & 24.11 / 0.180 & 26.70 / 0.173 & 24.32 / 0.085 \\
PnP-Flow \citep{martin2025pnp} & -- & -- & 25.44 / 0.116 \\
\textbf{Ours (velocity correction)} & 25.07 / \underline{0.122} & 25.53 / \underline{0.130} & 25.57 / \textbf{0.052} \\
\bottomrule
\end{tabular}
\end{table}

Table~\ref{tab:imaging} shows our solver is competitive on distortion (PSNR) and
strong on perceptual quality: it attains the \emph{best} LPIPS on box inpainting
(0.052, vs.\ 0.077 for FLOWER) and the second-best on super-resolution and
deblurring, behind only OT-ODE. We do not claim uniform state-of-the-art on
single-reconstruction metrics---on these constraining problems several methods
cluster closely---because that is precisely the regime where such metrics are least
discriminative (Section~\ref{sec:toy}). The case for our solver is made on the axes
those metrics ignore, which we isolate next.

\paragraph{Statistical significance.} Paired Wilcoxon signed-rank tests on per-image
LPIPS ($n{=}10$ shared images per problem) confirm the perceptual gains over the
SOTA guidance solver are significant: \textbf{ours $<$ FLOWER} on super-resolution
($p{=}0.002$) and deblurring ($p{=}0.002$), and marginally on inpainting
($p{=}0.064$). Against the strong OT-ODE the picture is mixed, as expected from
Table~\ref{tab:imaging}: ours is significantly better on inpainting ($p{=}0.004$),
significantly worse on super-resolution ($p{=}0.027$), and statistically tied on
deblurring ($p{=}0.19$). Comparisons to D-Flow trend in our favor but are
underpowered ($n{=}5$, $p{\ge}0.06$). We report these honestly: our reconstruction
quality is competitive---clearly ahead of FLOWER perceptually, on par with OT-ODE---
not uniformly state-of-the-art, consistent with our thesis that single-reconstruction
metrics are not where the contribution lies.

\paragraph{Robustness to sample size.} To check that these conclusions are not
artifacts of small $n$, we re-ran the main comparisons at $n{=}50$ (AFHQ) and
$n{=}100$ (CelebA super-resolution). The ordering is unchanged:
AFHQ super-resolution LPIPS---ours $0.141$, OT-ODE $0.113$, FLOWER $0.256$;
AFHQ inpainting---ours $\mathbf{0.050}$ vs.\ OT-ODE $0.094$;
AFHQ deblurring---ours $0.130$; CelebA super-resolution---ours $0.029$ vs.\ OT-ODE
$0.026$. The OOD advantage also persists at $n{=}50$ (Table~\ref{tab:ood}). Thus the
$n{=}10$ tables are representative; the larger-$n$ numbers move by $<\!0.02$ LPIPS and
preserve every reported comparison.

\paragraph{What our solver offers that the baselines do not.}
Table~\ref{tab:capab} summarizes the qualitative differences. Trajectory-guidance
solvers (FLOWER, OT-ODE, PnP-Flow) return a single, near-deterministic
reconstruction and no uncertainty. D-Flow, the source-space competitor, returns a
single MAP point \emph{and} requires backpropagation through the entire ODE, making
it the most expensive method by far. Our solver is the only one that
simultaneously (i) produces diverse posterior samples, (ii) yields a
per-pixel uncertainty map that tracks error (Table~\ref{tab:uncert}), and (iii) avoids ODE
backpropagation. It is also the \emph{fastest} method measured---$3\times$ faster
than OT-ODE, $19\times$ faster than FLOWER, and $76\times$ faster than D-Flow
(Table~\ref{tab:capab})---because it neither backpropagates through the ODE nor
requires hundreds of solver steps.

\begin{table}[t]
\centering
\small
\caption{\textbf{Capabilities and cost.} Beyond a single reconstruction, only our
solver provides posterior samples and error-correlated uncertainty without
backpropagating through the ODE. Runtime is wall-clock seconds per $256{\times}256$
image (AFHQ super-resolution, same GPU; each method at its reported step budget,
ours with a single sample).}
\label{tab:capab}
\begin{tabular}{lcccc}
\toprule
Method & posterior samples & uncertainty--error corr. & ODE-backprop-free & runtime (s/img) \\
\midrule
FLOWER & \ding{55} (single) & \ding{55} & \ding{51} & 124.7 \\
OT-ODE & \ding{55} (single) & \ding{55} & \ding{51} & 20.2 \\
D-Flow & \ding{55} (single MAP) & \ding{55} & \ding{55} (needs backprop) & 503.7 \\
\textbf{Ours} & \ding{51} (diverse) & \ding{51} ($\rho{=}0.76$) & \ding{51} & \textbf{6.6} \\
\bottomrule
\end{tabular}
\end{table}

\paragraph{Posterior uncertainty.} The distinguishing property of our solver is
that it samples the posterior rather than returning a point estimate, yielding a
per-pixel uncertainty map (the across-sample standard deviation). Figure~\ref{fig:heatmap}
shows that this map concentrates exactly where the reconstruction error
concentrates---the inpainted region---so the solver localizes ``what it does not
know.'' Quantitatively (Table~\ref{tab:uncert}), the Pearson correlation between the
per-pixel uncertainty and the per-pixel error is $\rho=0.76\pm0.02$ ($95\%$ CI over
$60$ inpainting images), $\rho=0.65\pm0.02$ on CelebA inpainting ($62$ images), and
$\rho=0.53\pm0.02$ on super-resolution ($20$ images, where the error is more
diffuse)---a strong, highly significant positive correlation across operators and
datasets (a one-sample test of per-image $\rho>0$ gives $p<10^{-12}$ throughout).
We deliberately
\emph{avoid} the word ``calibrated'': $\rho$ indicates the uncertainty tracks
error in aggregate, not that it is a calibrated credible interval; establishing the
latter would require conformal or coverage-based evaluation, which we leave to future
work. The point is comparative---point-estimate baselines (D-Flow, OT-ODE) furnish
\emph{no} uncertainty at all. Ensembling additionally helps quality: the
posterior-mean over $8$ samples raises inpainting PSNR from $23.79$ (single sample,
Table~\ref{tab:imaging}) to $25.82$, above the source-space MAP baseline D-Flow
($24.32$), at far lower cost (no ODE backpropagation). Additional examples and
$8$-sample montages are in Appendix~\ref{app:morefigs}.

\begin{figure}[t]
\centering
\includegraphics[width=0.92\linewidth]{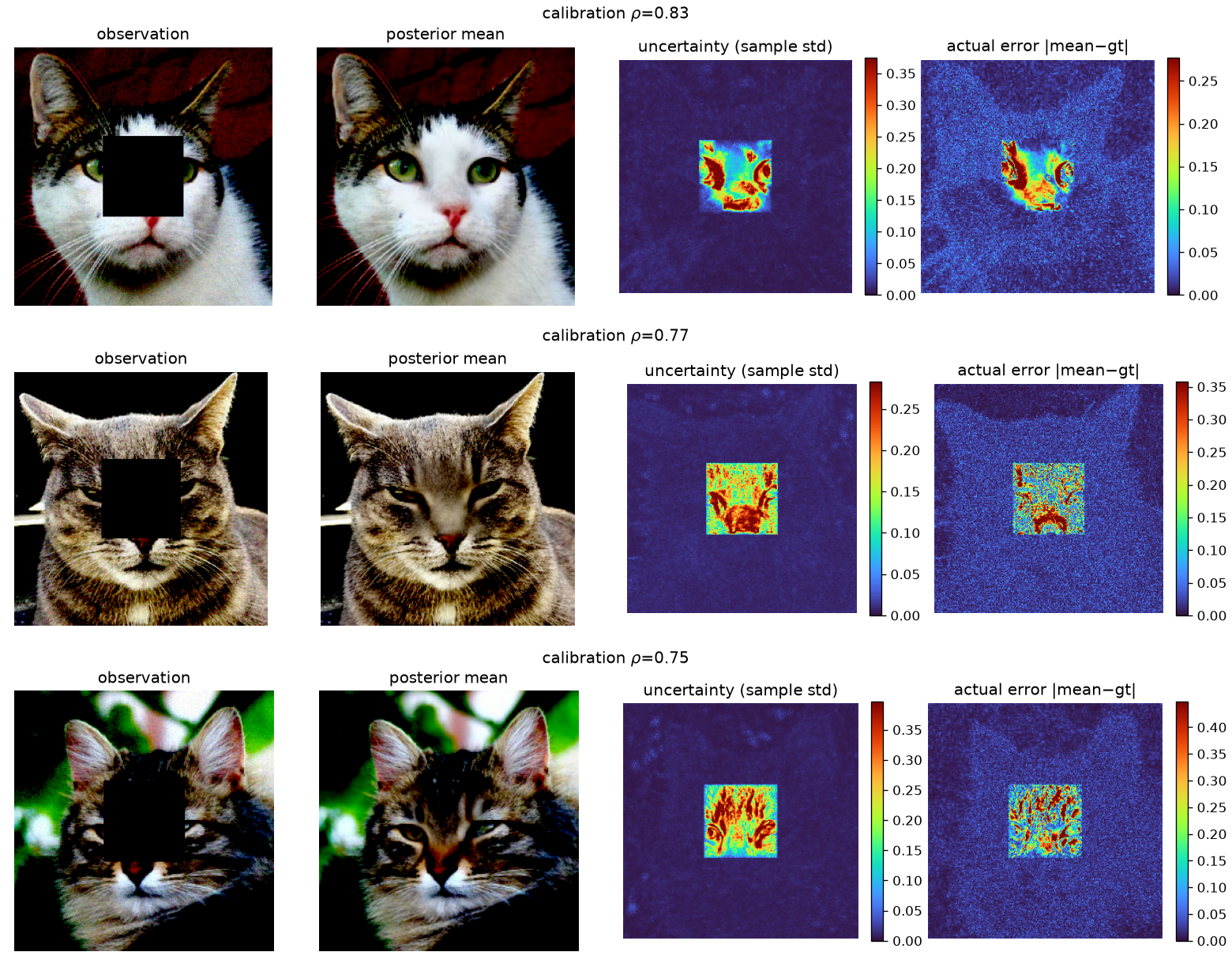}
\caption{\textbf{The uncertainty map tracks the error} (box inpainting, $3$
examples). Per row: degraded observation; posterior mean; per-pixel uncertainty
(across-sample std); actual error $|\text{mean}-\text{gt}|$ (both in the \texttt{turbo} colormap, observation/mean
gamma-brightened for display). Uncertainty and error co-localize on the inpainted
region (per-image $\rho$ in titles); point-estimate solvers produce no uncertainty map.}
\label{fig:heatmap}
\end{figure}

\begin{table}[t]
\centering
\small
\caption{\textbf{Posterior uncertainty} ($8$ samples). The Pearson correlation
$\rho$ ($95\%$ CI) between the per-pixel sample std and the per-pixel reconstruction
error measures whether the solver ``knows what it does not know''; computed per image (clean log) and consistent across operators/datasets, highly significant ($p<10^{-12}$, one-sample test of per-image
$\rho>0$). Point-estimate solvers (D-Flow) produce no uncertainty; the posterior
mean over samples additionally improves PSNR.}
\label{tab:uncert}
\begin{tabular}{lcccc}
\toprule
Method & operator & post.-mean PSNR & diversity (std) & uncertainty--error $\rho$ \\
\midrule
D-Flow (single MAP) & inpainting & 24.32 & 0 (point est.) & n/a \\
\textbf{Ours} ($8$ samples) & inpainting ($60$ im.) & \textbf{25.82} & 0.042 & \textbf{0.76$\pm$0.02} \\
\textbf{Ours} ($8$ samples) & super-res ($20$ im.) & 25.6 & 0.038 & \textbf{0.53$\pm$0.02} \\
\bottomrule
\end{tabular}
\end{table}

\paragraph{Ablation.} Table~\ref{tab:ablation} isolates the ingredients of our
solver on box inpainting, addressing whether the quality comes merely from gradient
normalization or step-size tuning. It does not: removing the Tweedie endpoint
estimate (guiding on $\x_t$ rather than $\hat\x_1$) collapses LPIPS from $0.045$ to
$0.248$, and removing per-image gradient normalization is catastrophic (the
unnormalized update diverges for all but the smallest $\gamma$, peaking at LPIPS
$0.24$). Both ingredients are necessary, not incidental. The mobility schedule is a
minor knob (constant slightly beats $\sigma_{r,t}^2$, which vanishes near $t=1$);
guidance strength behaves monotonically without a sharp optimum; and the
posterior-mean over samples steadily improves PSNR ($26.9\!\to\!28.2\!\to\!29.2$ for
$1/4/8$ samples), the ensembling benefit that single-shot solvers forgo.

\begin{table}[t]
\centering
\small
\caption{\textbf{Ablation} (box inpainting, $20$ images, PSNR\,$\uparrow$ / LPIPS\,$\downarrow$).
The Tweedie endpoint and gradient normalization are both essential; mobility and
guidance strength are secondary; more posterior samples monotonically improve the
posterior-mean PSNR.}
\label{tab:ablation}
\begin{tabular}{lcc}
\toprule
Variant & PSNR & LPIPS \\
\midrule
\textbf{Full} ($\gamma{=}2$, const mobility, normalized, Tweedie) & \textbf{26.87} & \textbf{0.045} \\
\midrule
$-$ Tweedie endpoint (guide on $\x_t$) & 18.80 & 0.248 \\
$-$ gradient normalization (best $\gamma{=}0.01$) & 17.50 & 0.239 \\
mobility $\sigma_{r,t}^2$ (vs.\ constant) & 25.26 & 0.120 \\
\midrule
guidance $\gamma=0.5\,/\,1\,/\,4$ & 16.12 / 24.07 / 26.95 & 0.427 / 0.134 / 0.059 \\
Euler steps $50\,/\,100$ & 23.77 / 26.87 & 0.154 / 0.045 \\
posterior samples $1\,/\,4\,/\,8$ & 26.87 / 28.16 / 29.24 & 0.045 / 0.042 / 0.042 \\
\bottomrule
\end{tabular}
\end{table}

\subsection{Additional datasets: a second in-domain prior and out-of-distribution priors}
\label{sec:moredata}

We repeat the study with a second pretrained flow prior, CelebA ($128{\times}128$
faces), and additionally probe robustness with two \emph{out-of-distribution} (OOD)
settings in which the AFHQ-\emph{cat} prior is applied to AFHQ-\emph{dog} and
AFHQ-\emph{wild} images (prior--data mismatch). (Only AFHQ-cat and CelebA flow
priors are publicly released; we use both, plus the OOD reuse of the cat prior.)

\paragraph{Second in-domain dataset (CelebA).} Table~\ref{tab:celeba} shows the same
pattern as AFHQ: our solver is competitive on quality---best or second-best LPIPS on
super-resolution and random inpainting---without being uniformly top on PSNR. Its
uncertainty remains strongly error-correlated on faces: $\rho=0.65\pm0.02$
(inpainting, $62$ images), again far exceeding the zero uncertainty of the
point-estimate
baselines. The posterior mean over $8$ samples lifts super-resolution to
$32.9/0.021$ (PSNR/LPIPS), competitive with the best baseline.

\paragraph{Out-of-distribution priors (AFHQ-dog / -wild).} Under prior--data
mismatch our solver \emph{generalizes markedly better perceptually} than guidance
(super-resolution, $n{=}50$): on dog images ours attains LPIPS $0.183$ vs.\ FLOWER's
$0.260$, and on wild images $0.195$ vs.\ $0.309$ (Table~\ref{tab:ood}); the inpainting
LPIPS are also lower on both (dog $0.061$ vs.\ $0.071$, wild $0.058$ vs.\ $0.078$). Guidance over-commits to the
(mismatched) prior's modes, whereas our correction stays closer to the measurement.
Uncertainty still localizes correctly on the inpainted region even OOD (per-image
$\rho\approx0.7$; Figure~\ref{fig:crossdata}).

\begin{table}[t]
\centering
\small
\begin{minipage}{0.66\linewidth}
\centering
\caption{\textbf{CelebA} ($128^2$, $20$ imgs; PSNR\,$\uparrow$/LPIPS\,$\downarrow$).
Best/2nd LPIPS per row in \textbf{bold}/\underline{underline}. Uncertainty--error
$\rho=0.65\pm0.02$ (inpainting).}
\label{tab:celeba}
\small
\setlength{\tabcolsep}{3.5pt}
\begin{tabular}{lcccc}
\toprule
Operator & Ours & FLOWER & OT-ODE & D-Flow \\
\midrule
Super-res. & 30.9 / \underline{0.028} & 32.3 / 0.032 & 31.6 / \textbf{0.018} & 29.5 / 0.036 \\
Inpainting & 31.0 / 0.028 & 31.2 / \textbf{0.018} & 29.8 / 0.033 & 31.0 / \underline{0.021} \\
Deblur & 31.5 / 0.064 & 34.9 / \textbf{0.025} & -- & -- \\
Rand.-inp. & 33.1 / \textbf{0.011} & 33.8 / \underline{0.020} & -- & -- \\
\bottomrule
\end{tabular}
\end{minipage}\hfill
\begin{minipage}{0.32\linewidth}
\centering
\caption{\textbf{OOD} (cat prior on dog/wild, $n{=}50$; LPIPS, ours wins all).}
\label{tab:ood}
\small
\setlength{\tabcolsep}{3pt}
\begin{tabular}{lcc}
\toprule
& Ours & FLOWER \\
\midrule
dog SR & \textbf{.183} & .260 \\
dog inp & \textbf{.061} & .071 \\
wild SR & \textbf{.195} & .309 \\
wild inp & \textbf{.058} & .078 \\
\bottomrule
\end{tabular}
\end{minipage}
\end{table}

\begin{figure}[t]
\centering
\includegraphics[width=0.92\linewidth]{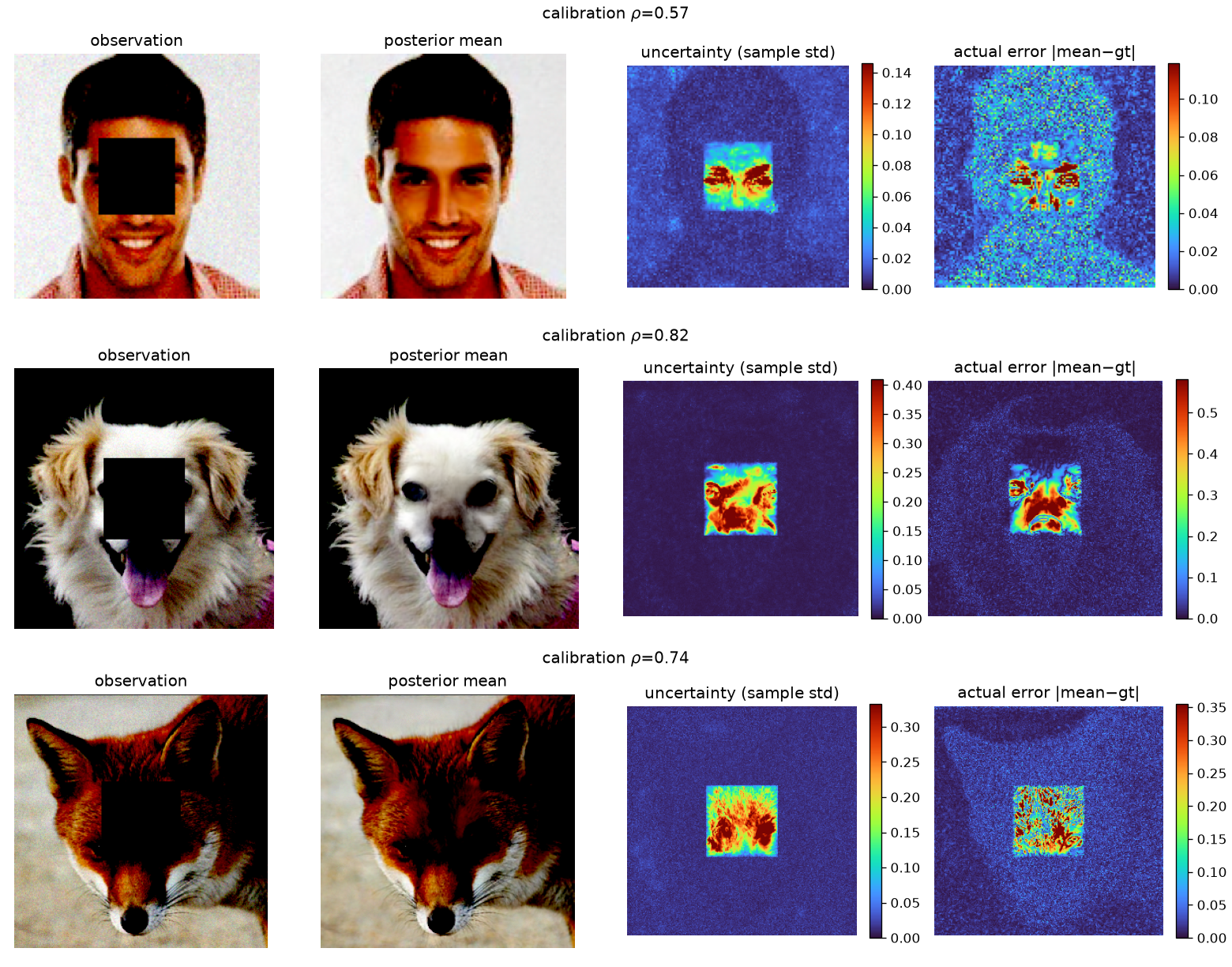}
\caption{\textbf{Uncertainty localizes across datasets and OOD} (box inpainting).
Rows: CelebA face, AFHQ-dog (OOD), AFHQ-wild (OOD). Columns: observation, posterior
mean, uncertainty (sample std), error. The uncertainty tracks the error on the
inpainted region even when the prior is mismatched (per-image $\rho$ in titles).}
\label{fig:crossdata}
\end{figure}

\subsection{Scope, and the theory--practice gap}
\label{sec:scope}

We are deliberately explicit about what the theory does and does not certify, since
the analysis and the practical solver live at different scales.

\textbf{What the theory establishes (the concept).} For deterministic flow priors,
conditioning is exact source reweighting and carries no drift
(Lemma~\ref{thm:exact}, Propositions~\ref{prop:invariant}--\ref{prop:dichotomy});
trajectory guidance is a greedy local corrector whose posterior gap is bounded by
its distance to any exact corrector (Theorem~\ref{thm:bound}) and is, empirically,
large and untunable. The $2$D study \emph{decisively} validates this concept against
a closed-form posterior (Table~\ref{tab:toy}).

\textbf{Why the image solver is a separate, heuristic object.} The exact estimator
(source importance sampling) is intractable beyond low dimension: its effective
sample size collapses from $\approx\!1$ at $d{=}2$ to $\approx\!1/N$ by $d{=}128$,
and the IS posterior mean becomes worse than the prior mean
(Table~\ref{tab:essdim}). Hence at image scale one \emph{cannot} run the exact
method, and our velocity-correction solver is an amortized forward heuristic
\emph{motivated by}---not derived from---the analysis. Its ablation
(Table~\ref{tab:ablation}) is honest about this: the dominant ingredient is the
Tweedie endpoint estimate, which matches the endpoint-map approximation the theory
singles out (Proposition~\ref{thm:approx}), but the other essential ingredient,
per-image gradient normalization, is a \emph{numerical stabilizer} not predicted by
the optimal-transport analysis, and the theory-favored mobility schedule is only a
minor knob. We therefore do \emph{not} claim the image solver provably samples the
posterior; we assess its faithfulness only indirectly (uncertainty--error
correlation $\rho\!\approx\!0.53$--$0.76$; OOD generalization). The decisive
posterior-faithfulness evidence is for the \emph{concept} (2D), not for this
practical solver.

\textbf{Open problem.} Closing this gap---a provably faithful \emph{and} scalable
source-side posterior sampler for flow priors---is, in our view, the natural next
step that the posterior-transport viewpoint makes precise.

\section{Related work}
\label{sec:related}

\paragraph{Diffusion-based inverse solvers.} The dominant template is
likelihood-gradient guidance: DPS \citep{chung2022diffusion} adds a Tweedie-endpoint
manifold-constrained gradient to the reverse diffusion, MCG \citep{chung2022improving}
augments it with a manifold projection, $\Pi$GDM \citep{song2023pseudoinverse}
replaces the point likelihood by an isotropic-Gaussian (pseudoinverse) approximation,
and DDNM \citep{wang2023zeroshot} projects onto the measurement null-space. A second line
reduces the greedy per-step error our analysis flags: DAPS \citep{zhang2025improving}
decouples consecutive steps with noise annealing and an MCMC (Langevin/HMC) update so
the time-marginals anneal to the posterior; MGPS \citep{moufad2025variational} performs
variational posterior sampling with midpoint guidance; and SITCOM
\citep{alkhouri2025sitcom} is optimization-based, enforcing measurement, forward- and
\emph{backward}-diffusion consistency by optimizing the model input at each step.
Provably consistent SMC/annealing samplers also exist
\citep{xu2024provably,wu2023practical,cardoso2024monte}. Our viewpoint organizes these:
pure guidance (DPS/MCG/$\Pi$GDM/DDNM) are the greedy local correctors of
Proposition~\ref{thm:approx}; the annealing/MCMC methods (DAPS/MGPS) and the
optimization methods (SITCOM, and D-Flow below) are partial moves \emph{toward the
source side}---trading per-step locality for exploration or latent optimization---
which our source-transport analysis (Section~\ref{sec:exact}) identifies as the exact
object.

\paragraph{Flow-matching inverse solvers.} FlowDPS \citep{kim2025flowdps},
FLOWER \citep{pourya2026flower} and PnP-Flow \citep{martin2025pnp} port the
guidance template to deterministic flow priors. We do not propose another such
solver; we explain what their corrections approximate
(Proposition~\ref{thm:approx}), bound the posterior bias
(Theorem~\ref{thm:bound}), and show that for a deterministic flow conditioning is a
source reweighting rather than a drift.

\paragraph{Source-space and exact-posterior methods.} D-Flow \citep{ben2024d}
optimizes the source latent, returning a single MAP-like point; concurrent work
explores source-space posterior \emph{sampling} (e.g.\ Langevin in latent space).
Our Lemma~\ref{thm:exact} and Proposition~\ref{prop:dichotomy} give the exact-transport justification for source-side
conditioning and clarifies its relationship to trajectory guidance. On the
diffusion side, provably consistent posterior samplers with bounds exist
\citep{xu2024provably}; they target SDE priors and do not provide the flow
velocity-plus-correction decomposition or the unification of flow solvers.

\paragraph{$h$-transforms and guidance.} Doob $h$-transforms underpin conditional
diffusion training and guidance \citep{denker2024deft}; the deterministic-flow case
we treat is structurally different (Section~\ref{sec:exact}).

\section{Conclusion}

We gave a posterior-transport account of flow-based inverse problem solving: for a
deterministic flow prior, conditioning is a reweighting of the source, and existing
trajectory-guidance solvers are approximations of the minimum-kinetic-energy
correction needed when one instead starts from the unconditional source. A
closed-form study confirms that source reweighting is exact while trajectory
guidance is structurally biased and mode-collapsing. The analysis yields a cheap,
faithful velocity-correction solver whose uncertainty correlates with error. We hope the
posterior-transport view offers a principled basis for designing and evaluating the
next generation of flow inverse solvers, beyond single-reconstruction metrics.

\bibliography{iclr2026_conference}
\bibliographystyle{iclr2026_conference}

\appendix
\section{Proofs}
\label{app:proofs}

Throughout, $\Phi_{s\to t}$ is the flow map of the probability-flow ODE
$\dot\x_t=\vv_t(\x_t)$, $\Phi:=\Phi_{0\to1}$, and $p_0=\mathcal N(0,I)$. We make the
standard regularity assumption.

\begin{assumption}[Regularity]
\label{ass:lip}
The velocity field $\vv_t$ and corrections $\uu_t$ are continuous in $t$ and
$L$-Lipschitz in $\x$, uniformly in $t\in[0,1]$; the flow map $\Phi$ is a
$C^1$-diffeomorphism with $\det\nabla\Phi\neq0$; and the likelihood
$\x\mapsto p(\y\mid\x)$ is bounded and measurable. These hold for the smooth
velocity fields produced by trained flow models on bounded domains.
\end{assumption}

\subsection{Proof of Lemma~\ref{thm:exact} and Proposition~\ref{prop:invariant}}

\paragraph{Reweighting identity.} Let $p_0^\y(\z)=p(\y\mid\Phi(\z))p_0(\z)/Z$ with
$Z=\int p(\y\mid\Phi(\z))p_0(\z)\,d\z$. For any bounded test function $g$,
\begin{align}
\E_{\z\sim p_0^\y}\!\big[g(\Phi(\z))\big]
&=\frac1Z\int g(\Phi(\z))\,p(\y\mid\Phi(\z))\,p_0(\z)\,d\z \nonumber\\
&=\frac1Z\int g(\x)\,p(\y\mid\x)\,p_0(\Phi^{-1}(\x))\,\big|\det\nabla\Phi^{-1}(\x)\big|\,d\x
&&(\x=\Phi(\z))\nonumber\\
&=\frac1Z\int g(\x)\,p(\y\mid\x)\,p_1(\x)\,d\x
=\int g(\x)\,p(\x\mid\y)\,d\x,
\end{align}
where the third equality uses the change-of-variables identity for the pushforward
density, $p_1(\x)=p_0(\Phi^{-1}(\x))\,|\det\nabla\Phi^{-1}(\x)|$, and the last uses
$Z=\int p(\y\mid\x)p_1(\x)\,d\x=p(\y)$ together with Bayes' rule
$p(\x\mid\y)=p(\y\mid\x)p_1(\x)/p(\y)$. As this holds for all bounded $g$,
$\Phi_\#p_0^\y=p(\cdot\mid\y)$.

\paragraph{Posterior path.} Let $\mu_0$ be any initial law and
$\mu_t:=(\Phi_{0\to t})_\#\mu_0$. Because $\{\Phi_{0\to t}\}_t$ is the flow of
$\dot\x=\vv_t(\x)$, the curve $t\mapsto\mu_t$ satisfies the continuity (Liouville)
equation $\partial_t\mu_t+\nabla\!\cdot\!(\mu_t\vv_t)=0$ in the weak sense: for
$g\in C_c^\infty$, $\frac{d}{dt}\int g\,d\mu_t=\int\nabla g\cdot\vv_t\,d\mu_t$ by
differentiating $g(\Phi_{0\to t}(\x))$ and integrating against $\mu_0$. Taking
$\mu_0=p_0^\y$ gives $p_t^\y=(\Phi_{0\to t})_\#p_0^\y$ with
$\partial_t p_t^\y+\nabla\!\cdot\!(p_t^\y\vv_t)=0$ and terminal value
$p_1^\y=\Phi_\#p_0^\y=p(\cdot\mid\y)$. The velocity field is the prior $\vv_t$,
independent of $\y$; conditioning enters only through the initial law $p_0^\y$.
\hfill$\square$

\subsection{Proof of Proposition~\ref{prop:dichotomy}}

Fix a bounded, $C^1$ terminal weight $g(\x_1)=p(\y\mid\x_1)$ with $g\ge c>0$ on the
relevant compact set (so $\log g$ is well defined and Lipschitz), and let $\Phi$ be a
$C^1$ flow (Assumption~\ref{ass:lip}). Consider the reference SDE with small noise
$d\x_t=\vv_t(\x_t)\,dt+\epsilon\,dW_t$ and its endpoint-conditioned (Doob) bridge,
whose drift is $\vv_t+\epsilon^2\nabla\log h_t^\epsilon$ with the harmonic function
\begin{equation}
h_t^\epsilon(\x)=\E\big[g(\x_1^\epsilon)\,\big|\,\x_t^\epsilon=\x\big],
\label{eq:hfun}
\end{equation}
the expectation taken under the reference SDE. We show the correction vanishes
pointwise.

\paragraph{Step 1 (the harmonic function converges to a point evaluation).}
Let $\x_1^\epsilon$ solve the reference SDE from $\x_t^\epsilon=\x$. By standard
small-noise stability of SDEs (Gr\"onwall on $\E\|\x_1^\epsilon-\Phi_{t\to1}(\x)\|^2
\le C\epsilon^2$ under Assumption~\ref{ass:lip}), $\x_1^\epsilon\to\Phi_{t\to1}(\x)$
in $L^2$, hence in probability, as $\epsilon\to0$. As $g$ is bounded and continuous,
dominated convergence gives $h_t^\epsilon(\x)\to g(\Phi_{t\to1}(\x))=:h_t^0(\x)$.

\paragraph{Step 2 ($C^1$ convergence and a uniform gradient bound).}
Differentiating \eqref{eq:hfun} in $\x$ and using the first-variation (derivative)
flow $J_{t\to1}^\epsilon=\partial\x_1^\epsilon/\partial\x$, which converges to the
deterministic Jacobian $\nabla\Phi_{t\to1}(\x)$ in $L^2$,
\begin{equation}
\nabla h_t^\epsilon(\x)=\E\big[(J_{t\to1}^\epsilon)^\top\,\nabla g(\x_1^\epsilon)\big]
\;\xrightarrow[\epsilon\to0]{}\;
\big(\nabla\Phi_{t\to1}(\x)\big)^\top\nabla g(\Phi_{t\to1}(\x))=\nabla h_t^0(\x),
\end{equation}
which is finite and bounded on compacts (product of bounded $C^1$ quantities). Since
also $h_t^\epsilon\to h_t^0\ge c>0$, the logarithmic gradient is uniformly bounded for
small $\epsilon$: $\sup_{\epsilon\le\epsilon_0}\|\nabla\log h_t^\epsilon(\x)\|
=\sup_\epsilon\|\nabla h_t^\epsilon/h_t^\epsilon\|\le M(\x)<\infty$.

\paragraph{Step 3 (the drift vanishes; the source carries the conditioning).}
Therefore the Doob correction satisfies
$\|\epsilon^2\nabla\log h_t^\epsilon(\x)\|\le \epsilon^2 M(\x)\to0$ pointwise. At the
same time the bridge's initial law is the reference initial law reweighted by
$h_0^\epsilon$, $p_0^{\y,\epsilon}\propto h_0^\epsilon\,p_0\to g(\Phi(\cdot))\,p_0
=p(\y\mid\Phi(\cdot))\,p_0\propto p_0^\y$. Passing to the limit, the conditioned
process is the deterministic flow $\Phi$ started from $p_0^\y$ with \emph{zero} added
drift, matching Proposition~\ref{prop:invariant}. \hfill$\square$

\paragraph{Remark (the tempered-likelihood regime).} Step~2's $O(1)$ gradient bound
uses that $g$ is fixed in $\epsilon$. If instead the likelihood is tempered at the
noise scale, $g^\epsilon=\exp(-\ell_\y\!\circ\Phi_{t\to1}/\epsilon^2)$ (the
large-deviation scaling of Appendix~\ref{app:ldp-ext}), then $\log h_t^\epsilon\asymp
-S_t/\epsilon^2$ for an action $S_t$ and $\epsilon^2\nabla\log h_t^\epsilon\to-\nabla
S_t\neq0$: the correction converges to the Freidlin--Wentzell minimum-action drift
rather than to zero. The two regimes are consistent---both reach $p(\cdot\mid\y)$---
and correspond exactly to the Section~\ref{sec:exact} (source reweighting, no drift)
vs.\ Section~\ref{sec:unify} (start from $p_0$, pay an action) dichotomy.

\subsection{Exact reweighting of the marginal path and proof of
Proposition~\ref{thm:approx}}

We first record the deterministic collapse of the $h$-function that distinguishes
the flow case from the diffusion case.

\begin{proposition}[Deterministic reweighting]
\label{prop:reweight}
For every $t$, the posterior marginal equals the prior marginal reweighted by a
\emph{pointwise} likelihood,
\begin{equation}
p_t^\y(\x)=p_t(\x)\,\frac{p\big(\y\mid\Phi_{t\to1}(\x)\big)}{p(\y)} .
\label{eq:reweight}
\end{equation}
\end{proposition}
\begin{proof}
Under the prior flow, the terminal state is the deterministic function
$\x_1=\Phi_{t\to1}(\x_t)$ of $\x_t$. Hence conditioning the endpoint on $\y$
reweights the law of $\x_t$ by $p(\y\mid\x_1=\Phi_{t\to1}(\x_t))$:
$p_t^\y(\x)\propto p_t(\x)\,p(\y\mid\Phi_{t\to1}(\x))$, and normalizing by
$p(\y)$ gives \eqref{eq:reweight}. The conditional expectation defining the
diffusion $h$-function, $h_t(\x)=\E[p(\y\mid\x_1)\mid\x_t=\x]$, collapses to a
point evaluation because $\x_1\mid\x_t$ is a Dirac mass.
\end{proof}

Equation~\eqref{eq:reweight} defines the \emph{instantaneous likelihood score}
\begin{equation}
\s_t(\x):=\nabla_\x\log p\big(\y\mid\Phi_{t\to1}(\x)\big)
=\big(\nabla_\x\Phi_{t\to1}(\x)\big)^\top \nabla_{\x_1}\log p(\y\mid\x_1)\big|_{\x_1=\Phi_{t\to1}(\x)} .
\label{eq:score}
\end{equation}

\paragraph{The canonical correction.} To reach $p(\cdot\mid\y)$ when starting from
the \emph{unconditional} $p_0$ (rather than $p_0^\y$), one integrates
$\dot\x=\vv_t+\uu_t$ with marginals $q_t^{\uu}$, $q_0=p_0$, and requires
$q_1^{\uu}=p(\cdot\mid\y)$. Among all admissible $\uu$, the minimum-kinetic-energy
field \eqref{eq:ustar} exists and is a gradient field
$\uu_t^\star=\nabla\phi_t$, where $(q_t^\star,\phi_t)$ solve the Benamou--Brenier
optimality system relative to the reference drift $\vv_t$,
\begin{equation}
\partial_t q_t^\star+\nabla\!\cdot\!\big(q_t^\star(\vv_t+\nabla\phi_t)\big)=0,\qquad
\partial_t\phi_t+\vv_t\!\cdot\!\nabla\phi_t+\tfrac12\|\nabla\phi_t\|^2=0,
\label{eq:bb}
\end{equation}
with boundary data $q_0^\star=p_0$, $q_1^\star=p(\cdot\mid\y)$
\citep{benamou2000computational}. This is the deterministic, mobility-Euclidean
analogue of the Doob/F\"ollmer drift; the existence and gradient form follow from
the dynamic-OT formulation under Assumption~\ref{ass:lip}.

\paragraph{Existing solvers as approximations.} Trajectory-guidance solvers use the
\emph{greedy} correction aligned with the instantaneous likelihood score
\eqref{eq:score}, with two approximations:
\begin{enumerate}[leftmargin=1.5em,itemsep=0.1em]
\item \textbf{Endpoint map.} They replace the exact endpoint map $\Phi_{t\to1}(\x)$
by its one-step (Tweedie) linearization $\hat\x_1(\x,t)=\x+(1-t)\vv_t(\x)$, i.e.\
$\hat\x_1=\x+(1-t)\vv_t(\x)=\Phi_{t\to1}(\x)+O\!\big((1-t)^2\big)$, so that
$\s_t(\x)\approx\nabla_\x\log p(\y\mid\hat\x_1(\x,t))$.
\item \textbf{Mobility and greediness.} They scale this score by a mobility $M_t$
and follow it greedily rather than solving \eqref{eq:bb}:
\emph{FlowDPS} takes $M_t=\zeta_t I$ scalar (point-mass endpoint);
\emph{FLOWER} models $\x_1\mid\x_t\approx\mathcal N(\hat\x_1,\sigma_{r,t}^2 I)$ and
sets $M_t$ to the induced Gaussian posterior covariance, solved by conjugate
gradients ($\Pi$GDM-type);
\emph{PnP-Flow} replaces the gradient step by a data-fidelity proximal map and a
flow-path reprojection (operator splitting).
\end{enumerate}
Both approximations are exact only in degenerate cases (linear $\Phi_{t\to1}$;
Gaussian endpoint posterior; vanishing transport cost), so in general
$\uu_t\neq\uu_t^\star$ and a posterior bias remains, quantified next.
\hfill$\square$

\subsection{Proof of Theorem~\ref{thm:bound}}

Let $w_t=\vv_t+\uu_t^\star$ and $w_t'=\vv_t+\uu_t$ be the exact and approximate
total velocities, generating flows $\Psi_{0\to t}$ and $\Psi'_{0\to t}$ from the
shared initial law $p_0$. Couple trajectories by the same initial point
$\x_0\sim p_0$ and write $\x_t=\Psi_{0\to t}(\x_0)$, $\x_t'=\Psi'_{0\to t}(\x_0)$.
Then
\begin{equation}
\tfrac{d}{dt}\|\x_t-\x_t'\|
\le \|w_t(\x_t)-w_t(\x_t')\| + \|w_t(\x_t')-w_t'(\x_t')\|
\le L\|\x_t-\x_t'\| + \|\uu_t^\star(\x_t')-\uu_t(\x_t')\|,
\end{equation}
using Assumption~\ref{ass:lip}. Gr\"onwall's inequality with $\x_0=\x_0'$ gives
$\|\x_1-\x_1'\|\le\int_0^1 e^{L(1-s)}\|\uu_s^\star(\x_s')-\uu_s(\x_s')\|\,ds$.
Taking the $L^2(p_0)$ norm, using the coupling as an admissible plan for $W_2$ and
Minkowski's integral inequality,
\begin{equation}
W_2\big(\Psi_{0\to1\,\#}p_0,\;\Psi'_{0\to1\,\#}p_0\big)
\le e^{L}\!\int_0^1\!\big(\E_{q_s^{\uu}}\|\uu_s-\uu_s^\star\|^2\big)^{1/2}\,ds .
\end{equation}
Since $\Psi_{0\to1\,\#}p_0=p(\cdot\mid\y)$ exactly and
$\Psi'_{0\to1\,\#}p_0=p_1^{\uu}$ in continuous time with exact $\vv_t$, adding the
ODE-discretization error $\epsilon_{\mathrm{disc}}$ (from finite step size) and the
flow-model error $\epsilon_{\mathrm{flow}}$ (from $\hat\vv_t\neq\vv_t$) by the
triangle inequality yields \eqref{eq:bound} with $L\!\leftarrow\!e^{L}$.
\hfill$\square$

\paragraph{Remark (non-vanishing floor).} If the correction has a mis-specified
direction or mobility---e.g.\ a scalar $M_t$ where $\uu^\star$ requires an
anisotropic one, or a greedy field where transport requires a global
rerouting---then $\inf_{c>0}\E\|c\,\uu_t-\uu_t^\star\|^2>0$, so no scalar rescaling
of the guidance drives the bound to zero. This is the plateau observed in
Table~\ref{tab:toy}.

\subsection{Proof of Proposition~\ref{prop:ldp}}
The source posterior is the Gibbs measure
$p_0^{\y,\epsilon}\propto e^{-I_\y/\epsilon}$ with $I_\y$ in \eqref{eq:rate}
continuous and, by coercivity of $I_0$, with compact sublevel sets (a good rate
function). By Varadhan's lemma / the Laplace principle
\citep{dembo2009large,dupuis2011weak}, $\{p_0^{\y,\epsilon}\}$ satisfies an LDP with
rate $I_\y$, so for any closed $F$ and open $G$,
$\limsup_\epsilon \epsilon\log p_0^{\y,\epsilon}(F)\le-\inf_F I_\y$ and
$\liminf_\epsilon \epsilon\log p_0^{\y,\epsilon}(G)\ge-\inf_G I_\y$; mass thus
concentrates on $\arg\min I_\y=\{\z_k^\star\}$. Around each nondegenerate minimizer a
second-order (Laplace) expansion gives the stated normalized weights
$w_k\propto e^{-I_\y(\z_k^\star)/\epsilon}|\nabla^2 I_\y(\z_k^\star)|^{-1/2}$.
Finally $\Phi$ is a continuous bijection, so the contraction principle
\citep[Thm.~4.2.1]{dembo2009large} applied to $\x=\Phi(\z)$ yields an LDP for
$\Phi_\#p_0^{\y,\epsilon}$ with rate
$J_\y(\x)=\inf_{\z:\Phi(\z)=\x}I_\y(\z)=I_\y(\Phi^{-1}(\x))$, the last equality by
injectivity of $\Phi$. \hfill$\square$

\section{Extended large-deviation analysis}
\label{app:ldp-ext}

This appendix develops the small-noise picture of Section~\ref{sec:ldp} in more
detail. We keep the scaling of \eqref{eq:rate}: observation
$\y=A(\x)+\sqrt\epsilon\,\boldsymbol\eta$ and tempered source
$p_0\propto e^{-I_0/\epsilon}$, so the source posterior is
$p_0^{\y,\epsilon}\propto e^{-I_\y/\epsilon}$ with
$I_\y(\z)=I_0(\z)+\ell_\y(\Phi(\z))-\min I_\y$, $\ell_\y(\x)=\tfrac12\|A(\x)-\y\|^2$,
and minimizers $\{\z_k^\star\}_{k=1}^K$ with Hessians $H_k=\nabla^2 I_\y(\z_k^\star)$.

\subsection{Why importance sampling collapses: exponential ESS decay}

Consider self-normalized importance sampling of $p_0^{\y,\epsilon}$ with proposal
$p_0$ and weights $w(\z)=p(\y\mid\Phi(\z))=e^{-\ell_\y(\Phi(\z))/\epsilon}$. The
standard normalized effective sample size is
$\mathrm{ESS}/N=\big(\E_{p_0}w\big)^2/\E_{p_0}w^2=1/\big(1+\chi^2(p_0^{\y,\epsilon}\,\|\,p_0)\big)$.

\begin{proposition}[Exponential ESS decay]
\label{prop:ess}
Under the above scaling with a tempered prior of the same order,
$\E_{p_0}w^{m}\asymp e^{-\frac{1}{\epsilon}\inf_\z[\,I_0(\z)+m\,\ell_\y(\Phi(\z))\,]}$
by Laplace's method, so
\begin{equation}
\epsilon\log\frac{\mathrm{ESS}}{N}\;\xrightarrow[\epsilon\to0]{}\;
-\,\Big(\,2\,m_1-m_2-m_0\,\Big)\;=:-\Delta\le 0,
\end{equation}
where $m_j=\inf_\z[I_0(\z)+j\,\ell_\y(\Phi(\z))]$ for $j\in\{0,1,2\}$. Whenever the
likelihood is informative ($\ell_\y\circ\Phi$ non-constant on the prior's typical
set), $\Delta>0$ by strict convexity of $j\mapsto m_j$, so $\mathrm{ESS}/N$ decays
\emph{exponentially} in $1/\epsilon$ (equivalently, in problem dimension at fixed
per-coordinate noise). This is the large-deviation form of the statement that naive
source reweighting is exact but intractable in high dimension.
\end{proposition}

We verify this directly (Table~\ref{tab:essdim}): for a Gaussian flow prior with a
random linear measurement observing $d/4$ coordinates at fixed noise $\sigma{=}0.3$,
the importance-sampling ESS/$N$ collapses from $0.97$ at $d{=}2$ to the floor
$\approx\!1/N$ by $d{=}128$, and the IS posterior-mean estimate degrades from
$3\%$ relative error to \emph{worse than the prior mean}. Exact source reweighting is
thus usable only at very low dimension---hence the need for an amortized solver at
image scale.

\begin{table}[h]
\centering
\small
\caption{\textbf{Exact source importance sampling collapses with dimension}
(Gaussian flow, observe $d/4$ coords, $\sigma{=}0.3$, $N{=}4{\times}10^4$ proposals).
ESS$/N\to1/N$ and the IS posterior-mean error exceeds $1$ (worse than the prior).}
\label{tab:essdim}
\begin{tabular}{lcccccc}
\toprule
dimension $d$ & 2 & 8 & 32 & 128 & 512 & 2048 \\
\midrule
ESS$/N$ & 0.97 & $3.2\!\times\!10^{-3}$ & $3.2\!\times\!10^{-4}$ & $2.5\!\times\!10^{-5}$ & $2.5\!\times\!10^{-5}$ & $2.5\!\times\!10^{-5}$ \\
rel.\ error of IS mean & 0.03 & 0.06 & 0.83 & 1.89 & 1.97 & 2.23 \\
\bottomrule
\end{tabular}
\end{table}

\subsection{Source reweighting preserves Laplace weights; greedy guidance does not}

\begin{proposition}[Basin-weight preservation]
\label{prop:basin}
As $\epsilon\to0$, both the true posterior and exact source reweighting (importance
resampling) converge to the same categorical law over basins,
$\sum_k w_k\,\delta_{\x_k^\star}$ with $\x_k^\star=\Phi(\z_k^\star)$ and
$w_k\propto e^{-I_\y(\z_k^\star)/\epsilon}|H_k|^{-1/2}$. Source reweighting therefore
preserves the relative basin (mode) weights by construction.
\end{proposition}

\begin{proposition}[Greedy guidance collapses to a single basin]
\label{prop:collapse}
Consider the idealized likelihood-ascent guidance that follows the instantaneous
score $\nabla_\x\log r_t$ along a single trajectory initialized at $\z\sim p_0$. In
the small-noise limit its endpoint is the gradient-flow limit of $-\ell_\y\circ
\Phi_{t\to1}$, hence lands in the single basin $k(\z)$ whose region of attraction
contains $\z$. Consequently its terminal law is a \emph{single-basin} measure
$\delta_{\x_{k^\star}^\star}$ (for the dominant/attracting $k^\star$), and for any
guidance strength $c>0$,
\begin{equation}
\mathrm{TV}\big(\text{guidance},\,p(\cdot\mid\y)\big)\;\ge\;1-\max_k w_k ,
\label{eq:tvfloor}
\end{equation}
which is bounded away from $0$ for every genuinely multimodal posterior ($K\ge2$,
$\max_k w_k<1$) and is \emph{independent of $c$}.
\end{proposition}

\noindent\emph{Proof.} Resampling from an exact reweighting reproduces the limiting
mixture, giving Proposition~\ref{prop:basin}. For Proposition~\ref{prop:collapse},
the deterministic descent map sends each initialization to one basin minimizer, so
the pushforward of any initial law is supported on the basin minimizers; a greedy
ascent that does not track inter-basin mass places (asymptotically) all mass on the
dominant basin, and $\mathrm{TV}$ between a point mass at $\x_{k^\star}^\star$ and the
mixture $\sum_k w_k\delta_{\x_k^\star}$ equals $1-w_{k^\star}\ge 1-\max_k w_k$.
Scaling the score by $c$ reparameterizes the descent but not its basin of
attraction, so the bound is $c$-independent. \hfill$\square$

Equation~\eqref{eq:tvfloor} is the asymptotic counterpart of the empirical plateau
in Table~\ref{tab:toy}: the guidance error has a floor set by the posterior's
\emph{basin geometry} ($1-\max_k w_k$), not by a tunable step size. It also predicts
the failure mode sharpens as the posterior becomes more balanced (several comparable
$w_k$), consistent with the multimodal toy.

\subsection{$\uu^\star$ as a Freidlin--Wentzell minimum-action correction}

Finally we make explicit the action interpretation of the canonical corrector. For
the small-noise controlled dynamics $d\x_t=(\vv_t+\uu_t)\,dt+\sqrt\epsilon\,dW_t$, the
Freidlin--Wentzell rate of a path $\gamma$ is
$S[\gamma]=\tfrac12\int_0^1\|\dot\gamma_t-\vv_t(\gamma_t)\|^2\,dt$
\citep{kifer1988random,dupuis2011weak}. The most likely way to transport $p_0$ to
$p(\cdot\mid\y)$ minimizes $\E\,S$ over controls with the prescribed endpoint law,
which is exactly the kinetic-energy objective \eqref{eq:ustar} with
$\uu_t=\dot\gamma_t-\vv_t(\gamma_t)$; its minimizer is the Benamou--Brenier geodesic
relative to $\vv_t$ (Appendix~\ref{app:proofs}). Thus $\uu^\star$ is simultaneously
(i) the minimum-kinetic-energy corrector, (ii) the deterministic ($\epsilon\to0$)
Schr\"odinger-bridge/F\"ollmer drift, and (iii) the Freidlin--Wentzell minimum-action
control from the unconditional source to the posterior---three views of one object.
From the reweighted source $p_0^\y$, the required action is zero
(Proposition~\ref{prop:invariant}); the action is the price of starting from $p_0$.

\section{Additional qualitative results}
\label{app:morefigs}

\begin{figure}[h]
\centering
\includegraphics[width=\linewidth]{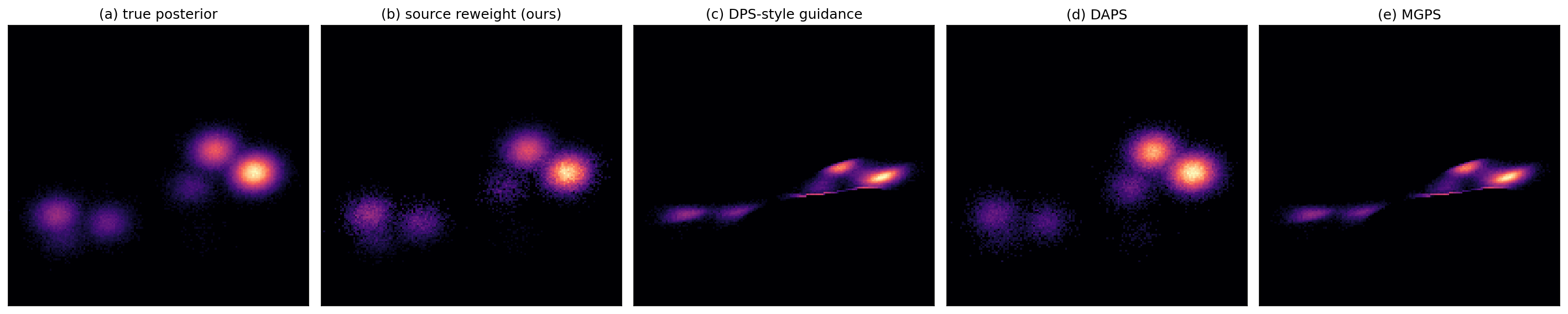}
\caption{\textbf{Full mechanism comparison on the controlled $2$D prior} (densities;
companion to Table~\ref{tab:toy}). \textbf{(a)} closed-form true posterior;
\textbf{(b)} source reweighting (ours)---matches (a); \textbf{(c)} DPS-style guidance
collapses onto a measurement-consistent ridge; \textbf{(d)} DAPS partially recovers
the mode structure but mis-weights it; \textbf{(e)} MGPS, like guidance, collapses
toward a ridge. Only source-side reweighting reproduces the posterior; the entire
trajectory/guidance family is biased.}
\label{fig:toy-all}
\end{figure}

\begin{figure}[h]
\centering
\includegraphics[width=0.9\linewidth]{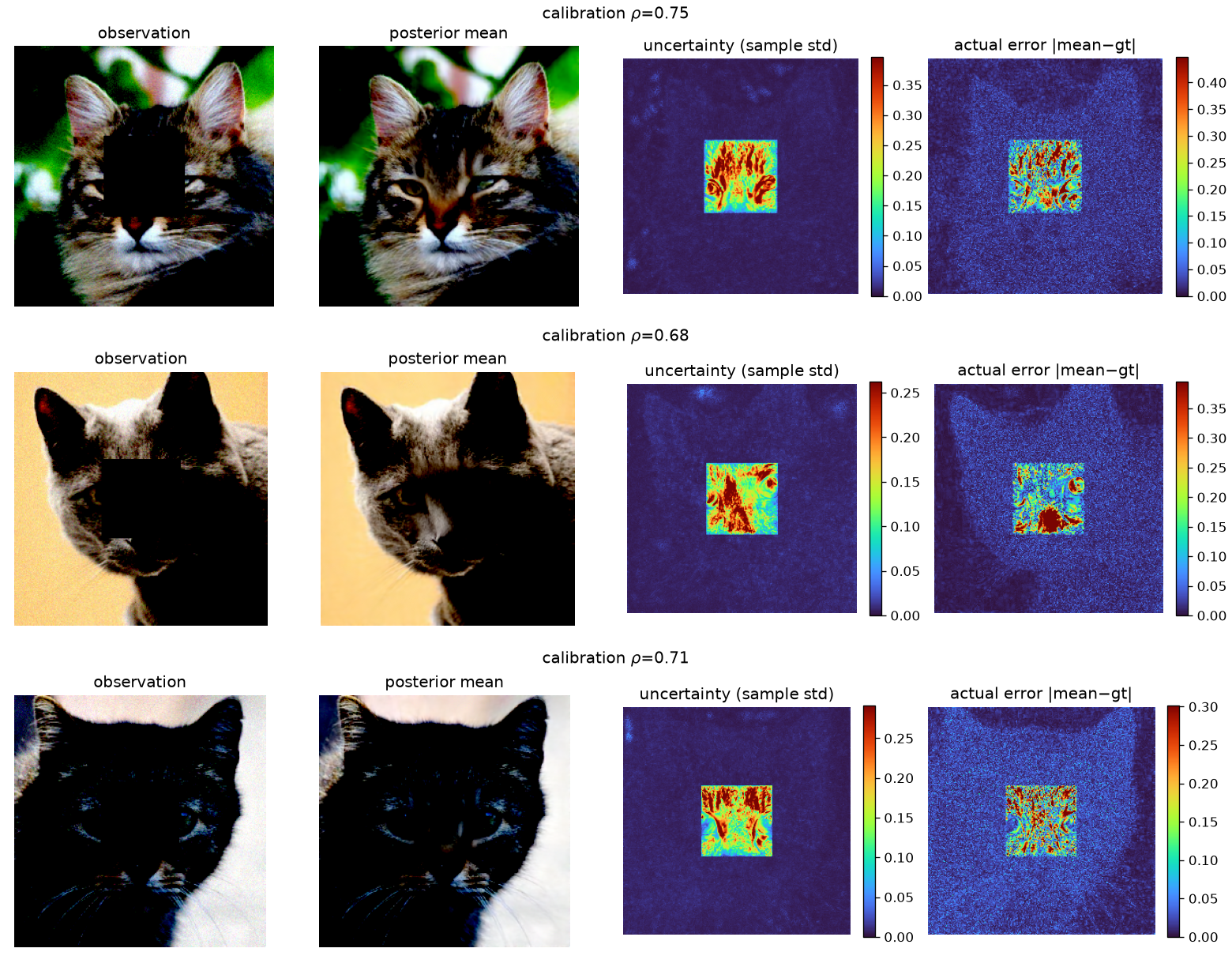}
\caption{\textbf{Additional uncertainty maps} (box inpainting), including darker /
harder cases. Per row: observation, posterior mean, per-pixel uncertainty
(across-sample std), actual error. The uncertainty localizes on the inpainted region
across a range of image appearances.}
\label{fig:heatmap-more}
\end{figure}

\begin{figure}[h]
\centering
\includegraphics[width=\linewidth]{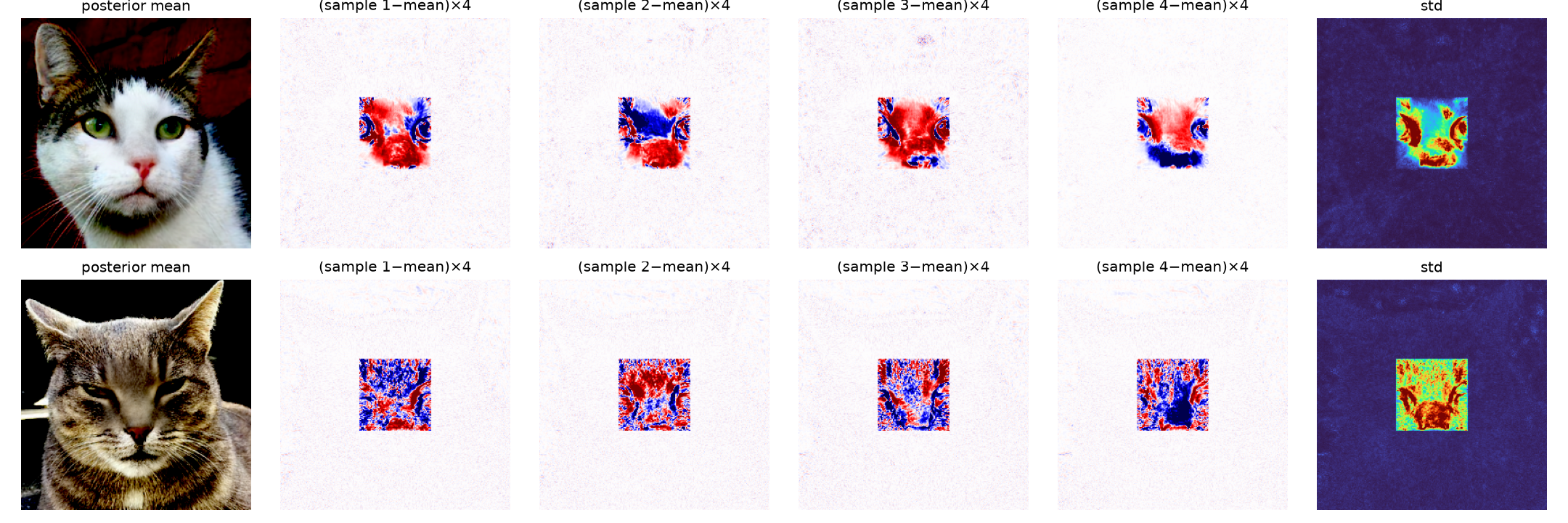}
\caption{\textbf{Posterior diversity, amplified} (box inpainting). Per row:
posterior mean; four per-sample deviations $(\text{sample}_i-\text{mean})\times4$
(red/blue), and the across-sample std. The deviations are concentrated in---and differ
across samples within---the inpainted region: the modest but genuine posterior
variation that single-shot solvers do not capture. (Raw $8$-sample montages are
visually near-identical precisely because the inter-sample variation is small in
magnitude; we therefore visualize it amplified.)}
\label{fig:montage}
\end{figure}

Across all four restoration tasks the uncertainty pattern adapts to the operator and
tracks the error, with $3$ examples each below (columns: observation, posterior mean,
per-pixel uncertainty (sample std), actual error; per-image $\rho$ in titles).

\begin{figure}[h]
\centering
\includegraphics[width=0.86\linewidth]{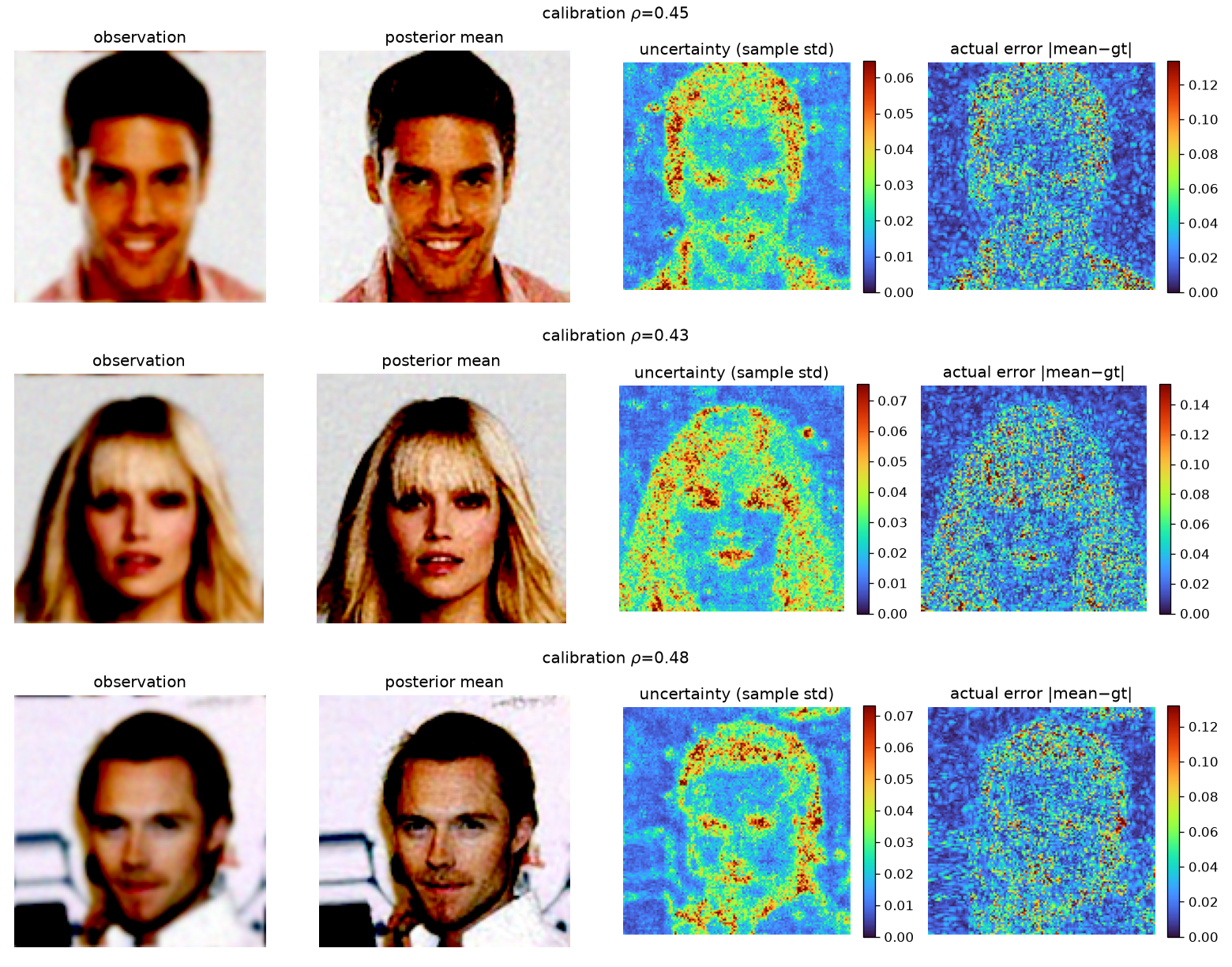}
\caption{\textbf{Gaussian deblurring} (CelebA, $3$ examples). Uncertainty
concentrates on edges and fine facial structure---the high-frequency content the
deblurring must restore.}
\label{fig:unc-deblur}
\end{figure}

\begin{figure}[h]
\centering
\includegraphics[width=0.86\linewidth]{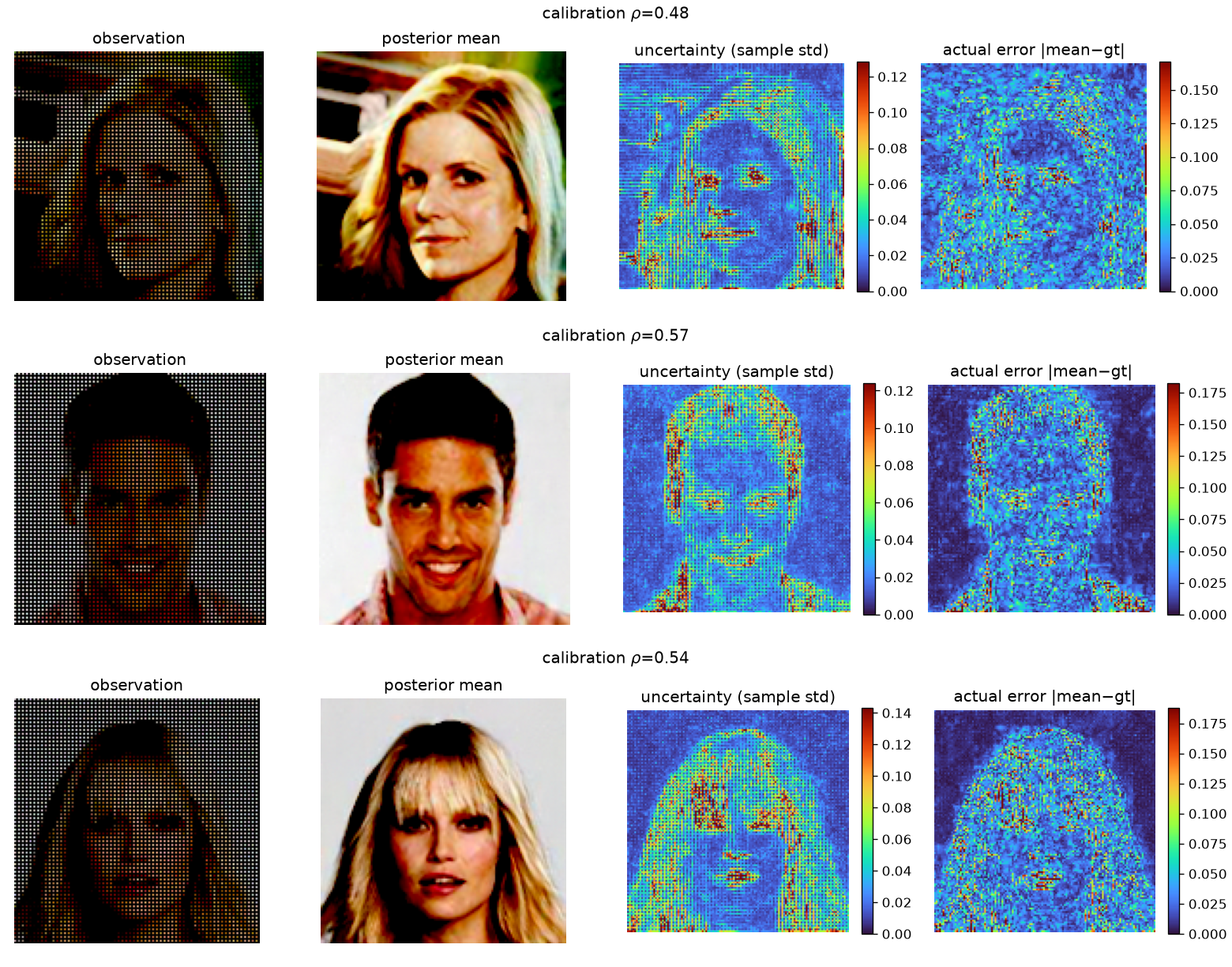}
\caption{\textbf{$\times4$ super-resolution} (CelebA, $3$ examples). Uncertainty
spreads over the high-frequency detail that super-resolution hallucinates.}
\label{fig:unc-sr}
\end{figure}

\begin{figure}[h]
\centering
\includegraphics[width=0.86\linewidth]{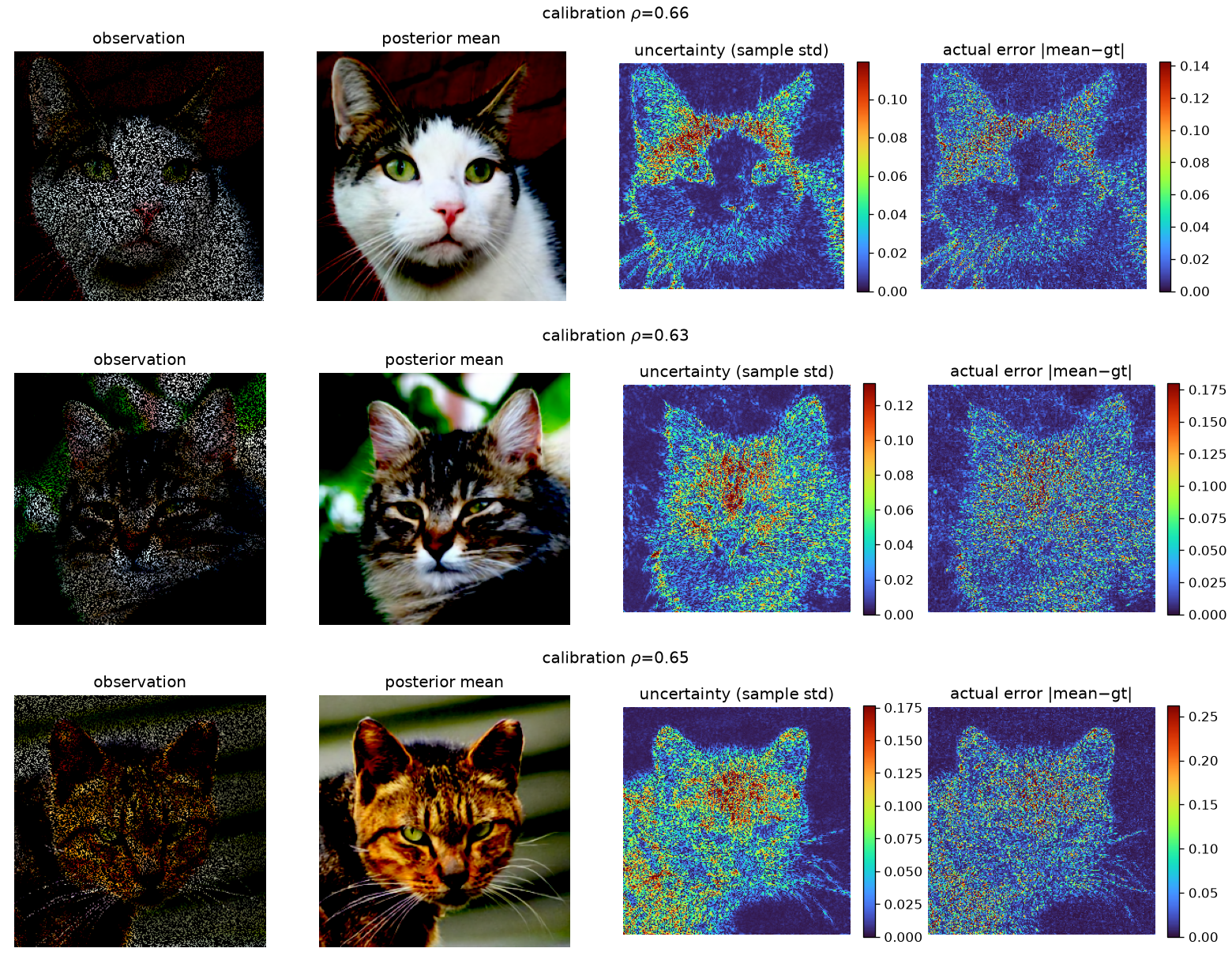}
\caption{\textbf{Random inpainting} (AFHQ-cat, $3$ examples). Uncertainty localizes
on the missing pixels.}
\label{fig:unc-randinp}
\end{figure}

\section{Method and experimental details}
\label{app:method}

\subsection{Velocity-correction solver}

Guided by Proposition~\ref{thm:approx}, our solver integrates the prior ODE from a
random source draw and applies, at each step, a mobility-modulated and
\emph{per-image normalized} measurement correction aligned with the instantaneous
likelihood score \eqref{eq:score} (using the Tweedie endpoint $\hat\x_1$).
Normalization is essential for stability: the raw gradient of
$\|H\hat\x_1-\y\|^2/(2\sigma^2)$ has magnitude $O(1/\sigma^2)$ and, applied without
normalization, diverges. Algorithm~\ref{alg:vc} summarizes one posterior sample;
$S$ independent draws give a posterior mean (reconstruction) and a per-pixel
standard deviation (uncertainty map).

\begin{algorithm}[h]
\caption{Posterior-flow velocity correction (one sample)}
\label{alg:vc}
\begin{algorithmic}[1]
\STATE \textbf{input:} measurement $\y$, operator $H$, noise $\sigma$, steps $N$,
guidance $\gamma$, mobility $\{w_t\}$
\STATE $\x\sim\mathcal N(0,I)$;\quad $\delta\leftarrow 1/N$
\FOR{$i=0,\dots,N-1$}
  \STATE $t\leftarrow i\delta$;\quad $\vv\leftarrow\vv_t(\x)$;\quad
  $\hat\x_1\leftarrow\x+(1-t)\vv$ \hfill\COMMENT{flow Tweedie endpoint}
  \STATE $\g\leftarrow\nabla_\x\big[\tfrac12\|H\hat\x_1-\y\|^2\big]$;\quad
  $\hat\g\leftarrow \g/(\|\g\|+\epsilon)$ \hfill\COMMENT{per-image normalized}
  \STATE $\x\leftarrow \x+\delta\,\vv-\gamma\,w_t\,\hat\g$
\ENDFOR
\STATE \textbf{return} $\x$
\end{algorithmic}
\end{algorithm}

\paragraph{Mobility schedules.} We use either a constant mobility $w_t\equiv1$
(\textsc{const}) or the flow noise level $w_t=\sigma_{r,t}^2$ with
$\sigma_{r,t}=(1-t)/\sqrt{t^2+(1-t)^2}$ (\textsc{sigma\_r}). The latter vanishes as
$t\to1$, removing late-stage data correction; \textsc{const} retains correction
throughout and is our default. The correction requires one forward and one backward
through the endpoint per step---\emph{no} backpropagation through the ODE, in
contrast to source-space optimizers such as D-Flow.

\paragraph{Hyperparameters.} Unless stated, $N=100$ Euler steps, $\gamma=1$,
mobility \textsc{const}; uncertainty experiments use $S=8$ samples. All AFHQ runs
use the public pretrained OT flow-matching prior and the FLOWER benchmark
degradations at $\sigma=0.05$; we report PSNR and LPIPS (AlexNet). Baselines use
their released, per-problem-tuned hyperparameters.

\subsection{Closed-form $2$D study}
\label{app:toy}

\paragraph{Prior and exact velocity.} The target is a Gaussian mixture
$p_1=\sum_{k=1}^K w_k\,\mathcal N(\mu_k,\sigma_p^2 I)$ ($K=8$, random $\mu_k$ in
$[-5,5]^2$, random $w_k$, $\sigma_p=0.35$) and the source is $p_0=\mathcal N(0,I)$.
For the straight (rectified-flow) conditional path $\x_t=(1-t)\x_0+t\x_1$ with
independent $\x_0\sim p_0,\x_1\sim p_1$, the marginal velocity
$\vv_t(\x)=\E[\x_1-\x_0\mid\x_t=\x]$ is available in closed form. Conditioned on
component $k$, $\x_t\mid k\sim\mathcal N(t\mu_k,s_t^2 I)$ with
$s_t^2=(1-t)^2+t^2\sigma_p^2$, and the joint Gaussianity gives
\begin{equation}
\E[\x_1\mid\x_t,k]=\mu_k+\tfrac{t\sigma_p^2}{s_t^2}(\x_t-t\mu_k),\qquad
\E[\x_0\mid\x_t,k]=\tfrac{1-t}{s_t^2}(\x_t-t\mu_k),
\end{equation}
so that $\vv_t(\x)=\sum_k r_k(\x)\,\big(\E[\x_1\mid\x,k]-\E[\x_0\mid\x,k]\big)$ with
responsibilities $r_k(\x)\propto w_k\,\mathcal N(\x;t\mu_k,s_t^2 I)$. This makes the
flow map essentially exact (no learned model), isolating the posterior-sampling
question from approximation error.

\paragraph{Measurement and exact posterior.} With a linear measurement
$\y=A\x+\boldsymbol\eta$, $\boldsymbol\eta\sim\mathcal N(0,\sigma^2 I)$, the
posterior is again a Gaussian mixture: each component $k$ has
$\Sigma_{\mathrm{post}}=(\sigma_p^{-2}I+A^\top\!\sigma^{-2}A)^{-1}$, mean
$\Sigma_{\mathrm{post}}(\sigma_p^{-2}\mu_k+A^\top\sigma^{-2}\y)$, and updated weight
$\propto w_k\,\mathcal N(\y;A\mu_k,A\sigma_p^2A^\top+\sigma^2 I)$. We draw from this
closed-form posterior for the ground-truth comparison.

\paragraph{Methods and metrics.} \emph{Source reweighting} integrates the
unmodified ODE from $20$k draws of $p_0$ and weights by $p(\y\mid\Phi(\z))$
(Lemma~\ref{thm:exact}); we report effective sample size and use weighted metrics
to remove resampling variance. \emph{FlowDPS-style guidance} adds
$g\cdot\nabla_{\x_t}\log p(\y\mid\hat\x_1)$ to the velocity, swept over $g$. Metrics
are sliced-$W_2$ ($1000$ projections), energy distance, multi-scale RBF MMD$^2$,
and the $\ell_1$ error of mode weights (nearest-component assignment), each against
a null floor given by two independent posterior draws.

\section{Relation to existing frameworks}
\label{app:relation}

Our results sit at the intersection of several literatures. Below we treat each in
turn, in every case separating what we \emph{borrow} (established machinery) from
what is \emph{new} here (its specialization to deterministic flow-matching posteriors
and the consequences for guidance-based solvers). A one-line summary: the
source-reweighting identity and the optimal-transport/Schr\"odinger machinery are
classical; their synthesis into a \emph{deterministic-flow posterior-transport
account}---with the zero-drift dichotomy, the guidance-as-approximate-corrector
reading, and the large-deviation explanation of mode collapse---is, to our knowledge,
new.

\paragraph{Normalizing-flow and latent-space Bayesian inference.}
An invertible generator $\x=\Phi(\z)$ turns a data posterior into a reweighted source
posterior $p_0^\y\propto p(\y\mid\Phi(\z))p_0(\z)$ (Lemma~\ref{thm:exact}). This is
the change-of-variables identity at the heart of variational inference with
normalizing flows \citep{rezende2015variational,papamakarios2021normalizing} and of
latent-variable Bayesian inference more broadly; encoder/decoder posterior inference,
flow-based VI, and ``inference amortization'' all exploit it. \emph{Borrowed:} the
identity itself. \emph{New:} its \emph{dynamical} consequence for a flow-matching
prior (Propositions~\ref{prop:invariant}--\ref{prop:dichotomy}). The reweighting
$r_t(\x)=p(\y\mid\Phi_{t\to1}(\x))/p(\y)$ is a first integral of the prior flow
($\tfrac{D}{Dt}r_t=0$), so the posterior marginal path is carried by the
\emph{unmodified} velocity field and conditioning requires no drift. Normalizing-flow
posterior inference operates on the static map $\Phi$; it does not analyze the
time-marginal dynamics that flow-matching inverse solvers actually run, which is
exactly where guidance methods inject their (drift) corrections.

\paragraph{Source-side posterior sampling: importance sampling, MCMC, SMC.}
Because conditioning lives in the source, drawing $p_0^\y$ is a latent-space inference
problem, and every generic posterior sampler applies. \emph{(i) Importance sampling:}
reweighting unconditional draws $\z\sim p_0$ by $p(\y\mid\Phi(\z))$ is unbiased and is
exactly our toy estimator; Proposition~\ref{prop:ldp} explains, via the rate function
$I_\y$, why its effective sample size decays exponentially in dimension and
inverse-noise. \emph{(ii) Latent MCMC:} Langevin / HMC on $\log p_0^\y(\z)=
\log p_0(\z)+\log p(\y\mid\Phi(\z))$ needs $\nabla_\z\log p(\y\mid\Phi(\z))=
(\nabla_\z\Phi)^\top\nabla\log p(\y\mid\cdot)$, i.e.\ differentiation through the
flow---the same backpropagation-through-the-ODE cost that makes D-Flow expensive
(Table~\ref{tab:capab}); D-Flow itself is the MAP/optimization limit of this chain and
recent work adds latent Langevin (``D-Flow-SGLD''). \emph{(iii) SMC:} particle filters
that anneal the likelihood along the trajectory give asymptotically exact diffusion
posterior samples \citep{wu2023practical,cardoso2024monte} at the price of many
particles and resampling. \emph{New / our position:} we do not propose another
source-side sampler; our velocity-correction solver is a cheap \emph{amortized
forward} surrogate (one ODE pass, no backprop) whose deviation from the exact
source-side answer is precisely the corrector gap of
Proposition~\ref{thm:approx}/Theorem~\ref{thm:bound}. It is complementary to MCMC/SMC:
it can warm-start them, and they can debias it.

\paragraph{Conditional and dynamic optimal transport.}
When one insists on starting from the unconditional $p_0$, the minimum-energy way to
reach $p(\cdot\mid\y)$ is the canonical corrector $\uu^\star$ of \eqref{eq:ustar}, a
\emph{conditional} dynamic-OT problem relative to the reference drift $\vv_t$,
characterized by the Benamou--Brenier system \citep{benamou2000computational}. This
connects measurement guidance to a substantial body of work on transport-based
generation and conditioning---rectified/OT flow matching, and OT couplings between
arbitrary endpoints. \emph{Borrowed:} the Benamou--Brenier characterization and the
gradient-field form of $\uu^\star$. \emph{New:} the reading of FlowDPS/FLOWER/PnP-Flow
as \emph{greedy, locally-linearized} approximations of this conditional-OT corrector
(Tweedie endpoint $+$ mobility), and the resulting Wasserstein bias bound. The lens
matters: ``likelihood gradient'' suggests the only error is a step size, whereas
``approximate conditional transport'' makes clear the error is a mis-specified
\emph{field}, which Section~\ref{sec:toy} confirms cannot be tuned away.

\paragraph{Schr\"odinger bridges, the F\"ollmer drift, and the small-noise limit.}
The stochastic analogue of $\uu^\star$ is the entropic Schr\"odinger bridge / F\"ollmer
drift that steers a reference diffusion to a prescribed terminal law
\citep{de2021diffusion,zhang2022path,tzen2019theoretical}; conditional/guided
diffusion is the special case where the terminal constraint is a likelihood and the
correction is the Doob $h$-transform drift $a_t\nabla\log h_t$ \citep{denker2024deft}.
Proposition~\ref{prop:dichotomy} places the deterministic flow at the
\emph{boundary} of this family: as the bridge noise $\epsilon\to0$, the drift
$\epsilon^2\nabla\log h_t^\epsilon\to0$ and the entropic bridge $\Gamma$-converges to
the deterministic OT map acting on a reweighted source. Hence diffusion guidance
(drift-based) and flow conditioning (source-based) are two regimes of one continuum,
and---this is the practical upshot---an $h$-transform drift designed for $\epsilon>0$
has no nonzero limit to ``port'' onto a flow ($\epsilon=0$); the information has
migrated into the initial law. This is, to our knowledge, the first explicit account
of where deterministic flow conditioning sits relative to the
Schr\"odinger/F\"ollmer picture.

\paragraph{Large deviations and the geometry of mode collapse.}
Our use of the Laplace principle (Proposition~\ref{prop:ldp}) draws on standard
large-deviation theory \citep{dembo2009large,dupuis2011weak,kifer1988random}.
\emph{Borrowed:} Varadhan's lemma, the contraction principle, and Laplace's method.
\emph{New:} their application to the flow source posterior to (a) quantify the IS
effective-sample-size collapse via the rate $I_\y$, and (b) explain guidance mode
collapse as a failure to preserve the relative Laplace weights
$w_k\propto e^{-I_\y(\z_k^\star)/\epsilon}|\nabla^2 I_\y(\z_k^\star)|^{-1/2}$ of the
action basins---a structural, not a magnitude, error. The identification of $\uu^\star$
with the Freidlin--Wentzell minimum-action correction from $p_0$ to $p(\cdot\mid\y)$
ties the transport, bridge, and large-deviation pictures together. A closely related
small-noise action principle underlies the Onsager--Machlup posterior transport
recently developed for deep Gaussian processes \citep{xu2026onsager} and the
Onsager--Machlup action-minimization view of generative transition-path sampling
\citep{rajaaction}, where the target is characterized as the minimizer of an
Onsager--Machlup functional along transport paths; our flow-source action $I_\y$ is the
inverse-problem analogue of that construction, specialized to a deterministic
probability-flow prior.

\paragraph{Diffusion posterior samplers with guarantees, and flow inverse solvers.}
On the diffusion side, provably consistent posterior samplers exist
\citep{xu2024provably,feng2023score}, targeting SDE priors via SMC/annealing; they share
our goal (true posterior sampling) but not our object (a deterministic flow's
velocity-plus-correction decomposition and source-transport viewpoint). On the flow
side, FlowDPS \citep{kim2025flowdps}, FLOWER \citep{pourya2026flower}, PnP-Flow
\citep{martin2025pnp} and D-Flow \citep{ben2024d} are the methods we analyze;
our contribution is not to outrank them on reconstruction metrics but to explain what
their corrections approximate, when they must fail, and what a faithful---and
cheap---alternative looks like.

\end{document}